\documentclass{article}

\usepackage{microtype}
\usepackage{graphicx}
\usepackage{subcaption}
\usepackage{booktabs} 

\usepackage{hyperref}

\usepackage[nonatbib,preprint]{neurips_2026}
\usepackage[numbers,sort&compress]{natbib}

\usepackage{amsmath}
\usepackage{amssymb}
\usepackage{mathtools}
\usepackage{amsthm}
\usepackage{mathmacros}
\usepackage{xspace}
\usepackage{xcolor}
\usepackage{thmtools, thm-restate}

\usepackage[capitalize,noabbrev]{cleveref}

\theoremstyle{plain}

\theoremstyle{definition}

\theoremstyle{remark}

\usepackage[textsize=tiny]{todonotes}

\title{High dimensional theory of two-phase optimizers\aga{comments on}}

\author{Atish Agarwala\\
Google DeepMind
}

\newif\ifcomments
\commentsfalse 

\ifcomments
\newcommand{\TODO}[1]{{\color{magenta}[TODO] #1}}
\newcommand{\aga}[1]{{\color{teal}[AA] #1}}
\newcommand{\krnote}[1]{{\color{purple}[KR] #1}}
\else
\newcommand{\TODO}[1]{}
\newcommand{\aga}[1]{}
\newcommand{\krnote}[1]{}
\fi

\newcommand{\diloco}{DiLoCo\xspace}
\newcommand{\la}{LA\xspace}
\newcommand{\sla}{SLA\xspace}
\newcommand{\ladiloco}{\la-\diloco}

\newcommand{\y}{\m{y}}
\newcommand{\z}{\m{z}}
\newcommand{\f}{\m{f}}
\newcommand{\zz}{z}

\newcommand{\J}{\m{J}}
\newcommand{\G}{\m{G}}

\newcommand{\dl}{\delta}

\renewcommand{\th}{\sm{\theta}}
\newcommand{\psb}{\sm{\psi}}
\newcommand{\Lo}{\mathcal{L}}
\renewcommand{\P}{P}
\newcommand{\D}{D}
\newcommand{\lr}{\eta}

\renewcommand{\H}{\m{H}}
\newcommand{\g}{\m{g}}

\renewcommand{\u}{\m{u}}
\newcommand{\V}{\m{V}}

\newcommand{\ntk}{\hat{\Theta}}

\newcommand{\B}{B}

\newcommand{\Id}{\m{I}}

\newcommand{\tl}{\tilde}

\newcommand{\tz}{\tl{\z}}

\newcommand{\TT}{S}

\newcommand{\xx}{x}
\newcommand{\R}{R}
\newcommand{\lam}{\lambda}

\newcommand{\diag}{{\rm diag}}

\newcommand{\pmat}{\m{P}}

\newcommand{\covm}{\sm{\Sigma}}
\newcommand{\pvec}{\m{p}}
\newcommand{\lmat}{\sm{\Lambda}}

\newcommand{\T}{\m{T}}

\newcommand{\bmat}{\m{B}}
\newcommand{\mmat}{\m{M}}

\newcommand{\amat}{\m{A}}

\newcommand{\lafrac}{\nu}
\newcommand{\mum}{\sm{\mu}}
\newcommand{\kvec}{\m{k}}
\newcommand{\tk}{\tl{\kvec}}
\newcommand{\kk}{k}

\newcommand{\bone}{\beta_{1}}

\newcommand{\zetvec}{\sm{\zeta}}
\newcommand{\om}{\omega}
\newcommand{\bin}{\beta_{\rm in}}
\newcommand{\bout}{\beta_{\rm out}}
\newcommand{\Tin}{\T_{\rm in}}
\newcommand{\Tout}{\T_{\rm out}}
\newcommand{\Tsyn}{\T_{\rm sync}}
\newcommand{\Ttot}{\T_{\rm tot}}

\newcommand{\mF}{\m{F}}
\newcommand{\mG}{\m{G}}
\newcommand{\mH}{\m{H}}
\newcommand{\FF}{F}

\newcommand{\tlam}{\tl{\lam}}

\newcommand{\convr}{r}
\newcommand{\mux}{\mu_{x}}
\newcommand{\muy}{\mu_{y}}
\newcommand{\yy}{y}
\newcommand{\dd}{d}
\newcommand{\hkk}{\hat{\kk}}
\newcommand{\zema}{\hat{\zz}}
\newcommand{\bema}{\beta_{\rm ema}}
\newcommand{\mm}{\m{m}}
\newcommand{\tmm}{\tilde{\m{m}}}
\newcommand{\tmat}{\hat{\m{T}}}

\begin{document}

\maketitle




\begin{abstract}
The trend towards larger training setups has brought a renewed interest in partially asynchronous two-phase optimizers which optimize locally and then synchronize across workers. Additionally, recent work suggests that the one-worker version of one of these algorithms, \diloco, shows promising results as a (synchronous) optimizer. Motivated by these studies we present an analysis of \ladiloco, a simple member of the \diloco family, on a high-dimensional linear regression problem. We show that the one-worker variant, \la, provides a different tradeoff between signal and noise than SGD, which is beneficial in many scenarios. We also show that the multi-worker version generates more noise than the single worker version, but that this additional noise generation can be ameliorated by appropriate choice of hyperparameters. We conclude with an analysis of \sla -- \la with momentum -- and show that stacking two momentum operators gives an opportunity for acceleration via a non-linear transformation of the ``effective'' Hessian spectrum, which is maximized for Nesterov momentum. Altogether our results show that two-phase optimizers represent a fruitful new paradigm for understanding and improving training algorithms.
\end{abstract}

\section{Introduction}

The success of deep learning has brought a pressing need for training algorithms that operate at scale, due to the desire to train larger models with larger and more complex datasets for longer
time. Most progress in this area has been made by synchronous systems, which compute gradients across multiple machines but sync them before changing parameters.
However, the cost of operating synchronous systems can grow superlinearly with scale due to
issues like memory, communications bandwidth, and hardware failures.

This has lead to numerous attempts to develop asynchronous training systems. While fully asynchronous training remains challenging,
an intermediate approach is to design distributed optimizers where
groups of workers maintain in-group synchrony, but between group synchronization takes a longer time \cite{mcmahan2017communication, charles2021convergence}.
This allows for more flexible hardware configurations, and the possibility for more compute to be brought to bear on a single training run.

Recent work
suggests that a two-phase optimization framework is a promising approach \citep{reddi2021adaptive, malinovskiy2020local,charles2021convergence,charles22iterated,khaled2025understanding}. Workers first perform an \emph{inner optimization} using
algorithms (e.g. SGD), and generate a \emph{pseudo-gradient} --- the difference between the parameters at the start and end of the inner optimization. These pseudo-gradients are then
averaged and sent to an \emph{outer optimizer} (e.g. SGD momentum) and used to perform a single optimization step --- the results of which are sent to the workers as the starting
point for the new inner optimization.

One particular configuration of this two-phase framework, \diloco, exhibits promising results in distributed LLM training \citep{douillard2023diloco}, a fact which has since been repeatedly observed empirically~\citep{jaghouar2024opendilocoopensourceframeworkglobally,douillard2025streaming,charles2025communication}, but the theoretical underpinnings of this performance remain mysterious. Even more surprising is the fact that \diloco shows benefits in the one-worker
(synchronous) setting \citep{charles2025communication}. There is even less theoretical understanding of this phenomenon, with some conjectures that the benefits are due to
implicit regularization of the algorithms \citep{kallusky2025snoo}.

In this work we provide theoretical evidence that two-phase algorithms like \diloco
provide \emph{direct benefits to optimization itself}. We use tools from high-dimensional learning dynamics
theory to exactly analyze the most basic algorithm in this class, Lookahead (\la) \citep{zhang2019lookahead},
as well as the multi-worker version \ladiloco (with outer optimizer SGD),
and write an evolution equation for the expected value of the loss, which concentrates in high dimensions.
We show the following:
\begin{itemize}
\item Two-phase optimizers induces a noise structure that is quantitatively
different from SGD.
\item \la allows a different tradeoff between signal and noise than SGD, which is provably better than SGD at optimal learning rates.
\item \ladiloco generates more noise in the distributed setting than the synchronous setting,
but there are hyperparameter choices which minimize the discrepancy between the methods
for small number of workers.
\end{itemize}
We then provide analysis of the extension Super Lookahead (\sla) --- \la with momentum in the
inner and outer optimizer. In the deterministic setting, we prove some key features of
stacking momentum operators:
\begin{itemize}
\item The outer optimizer operates on an ``effective spectrum'' generated by the inner optimizer, non-linearly transforming the relationship between Hessian spectrum and
convergence rate.
\item Nesterov-style momentum has large benefits over heavy ball momentum in this setting.
\end{itemize}
Our analysis suggests that there is a rich family of behaviors induced by \la and \sla which
can't be captured by more traditional algorithms and is worth further experimental and theoretical study.

\section{Basics of high-dimensional linear regression}

We introduce the high-dimensional linear regression setup studied previously~\citep{lee2022trajectory, agarwala2024high}.
We will present known results for the dynamics of SGD in this setting, so we can build on them
to analyze the \la optimizer.
Consider a scenario with $\D$ targets $\y_{tr}$ that we are trying to fit with a linear
model with a $\P$ dimensional parameter vector $\th$ by minimizing MSE loss:
\begin{equation}
\f(\th) = \y_{0}+\J\th,~\Lo(\th) = \frac{1}{2\D}||\z||^{2},~\z\equiv \f(\th)-\y_{tr}.
\end{equation}
Here $\J$ is the $\D\times\P$ dimensional Jacobian, and $\z$ are the residuals.
Previous works studied mini-batch SGD, where $\B$ examples are sampled from the $\D$ total
datapoints in order to compute the gradient. The update rule for SGD in this setting
is given by
\begin{equation}
\th_{t+1}-\th_{t} = -\frac{\lr}{\B} \J^{\tpose}\pmat_{t}\z_{t}.
\end{equation}
where $\pmat_{t}$ is a $\D\times\D$ projection matrix, diagonal with $\B$ entries with value
$1$, and the rest $0$, chosen at random. The $\pmat_{t}$ are i.i.d. for different $t$.
One feature of this model is that the updates close in function space.
Defining the empirical NTK $\ntk\equiv \frac{1}{\D}\J\J^{\tpose}$, the residuals evolve as
\begin{equation}
\z_{t+1}-\z_{t}  = -\lr\frac{\D}{\B}\ntk\pmat_{t}\z_{t}.
\label{eq:lin_reg_sgd}
\end{equation}

The dynamics of the first and second moments can also be computed exactly
(see Equation 6 and 7 and Appendix A.2 of \citep{agarwala2024high}).


For MSE loss,
computing the diagonal of the covariance $\expect_{\pmat}[\z_{t+1}\z_{t+1}^{\tpose}]$
would give us the expected loss trajectory. In the high dimensional limit
($\B\gg 1$, $\D\gg1$, $\B/\D = O(1)$), \citet{lee2022trajectory} showed that under appropriate
conditions on $\ntk$, the dynamics of the
second moment simplify considerably and the loss trajectories concentrate. We recap the concentration results in the notation of \citet{agarwala2024high}, to ease later exposition.
Given the
eigendecomposition $\ntk = \V\lmat\V^{\tpose}$, where $\V$ is orthogonal and $\lmat$ is diagonal and PSD,
we define the normalized diagonal $\pvec_{t} \equiv \lmat^{+}\diag(\V^{\tpose}\expect_{\pmat}[\z_{t}\z_{t}^{\tpose}]\V)$. The loss is $\Lo(\th) = \frac{1}{\D}\m{1}^{\tpose}\lmat\pvec$. $\pvec$ evolves as
\begin{equation}
\pvec_{t+1}  = (\amat+\bmat)\pvec_{t},~ \amat  \equiv(\Id-\lr\lmat)^2,
 ~\bmat\equiv \left(\B^{-1}-\D^{-1}\right)\lr^2\lmat\m{1}\m{1}^{\tpose}\lmat
\label{eq:pvec_high}
\end{equation}
in the high-dimensional limit.
The $\amat$ term corresponds to the deterministic dynamics of full batch training, and
the $\bmat$ term controls the noise, which is larger for small batch size
$\B$.
This system of equations can be used to derive quantities like the stochastic edge of stability --- the critical learning rate above which dynamics diverges due to noise
accumulation.

\section{Stochastic dynamics of \ladiloco}

\aga{Add full derivations in appendix if needed.}

We use the model from the previous section and extend the analysis to
understand \ladiloco. We first study \la, the one-worker case of \ladiloco, and prove that in this
setting the two-tier dynamics of \la can reduce the loss faster than SGD with optimal learning rate. We then study the multi-worker case, and analyze how the number of workers
reshapes the dynamics of the noise.

We first analyze the lookahead (\la) algorithm, first introduced in \citet{zhang2019lookahead}. This is one of the
most basic training algorithms compatible with the \diloco framework,
but has only one worker.
Given a sync time $\TT$, the update rule is given by:
\begin{equation}
\psb_{t} := \th_{t}{\rm~(when~} t//\TT=0)
\end{equation}
\begin{equation}
\psb_{t+1}-\psb_{t} = -\lr\g_{t}(\psb_{t}){\rm~for~}\TT{\rm~steps}
\end{equation}
\begin{equation}
\th_{t+\TT}-\th_{t} = \lafrac(\psb_{t+\TT}-\th_{t})
\end{equation}
where $\g_{t}$ is the minibatch gradient.
In other words, there is an inner optimizer that applies SGD to a copy
$\psb$ of $\th_{t}$ for $\TT$ steps
with \emph{inner learning rate} $\lr$,
after which an outer optimizer applies SGD based on the difference $\psb_{t+\TT-1}-\th_{t}$ with \emph{outer learning rate} $\lafrac$. For $\lafrac = 1$ this corresponds to regular SGD,
and for $\lafrac = 0$ this corresponds to no change to the parameters at all.
In the high dimensional linear model, the gradient is $ \J^{\tpose}\pmat_{t}\z_{t}(\psb_{t})/\B$ and the algorithm can again be written entirely in function space:
\begin{equation}
\tz_{t} := \z_{t}{\rm~(when~} t//\TT=0)
\end{equation}
\begin{equation}
\tz_{t+1}-\tz_{t} = -\lr\frac{\D}{\B}\ntk\pmat_{t}\tz_{t}{\rm~for~}\TT{\rm~steps}
\end{equation}
\begin{equation}
\z_{t+\TT}-\z_{t} = \lafrac(\tz_{t+\TT}-\z_{t})
\end{equation}
\begin{equation}
\z_{t}\equiv f(\th_{t})-\y_{tr},~\tz_{t}\equiv f(\psb_{t})-\y_{tr}
\end{equation}

Using these equations, we can compute the dynamics of the second moment eigenmodes in the high-dimensional limit of \citet{lee2022trajectory} (Appendix \ref{app:pvec_derivation}):
\begin{equation}
\pvec_{t+\TT}  = \left([(1-\lafrac)\Id+\lafrac(\Id-\lr\lmat)^{\TT}]^{2}+\right.
 \left.\lafrac^{2}[\T_{\TT}-(\Id-\lr\lmat)^{2\TT}]\right)\pvec_{t},~\T_{k}\equiv(\amat+\bmat)^{k}
\label{eq:look_ahead_pvec}
\end{equation}
The first term corresponds to the the deterministic dynamics of \la. The second term
can then be interpreted as the dynamical contribution of noise after $\TT$ steps of SGD,
and generally slows convergence rate relative to the deterministic setting.

We immediately see that \la allows a different path for trading off progress on the objective from the deterministic terms and the generation of noise.
In SGD, increasing $\lr$ speeds up the convergence encoded by the $\amat$ matrix but also
increases the noise encoded by the $\bmat$ matrix, and an optimal choice of $\lr$ trades
these off efficiently. In \la we have the additional option to tune the outer learning rate $\lafrac$
with similar consequences; increasing it improves deterministic
progress but also increases noise. This gives us a two-dimensional space of tradeoffs between deterministic progress and noise generation. Increase in noise from a larger
inner learning rate $\lr$ can be offset by a smaller outer learning rate $\lafrac$ and vice
versa.

For small learning rates these effects are redundant; if $||\bmat||_{\rm op}\ll||\amat||_{\rm op}$ and $\lr\lam\TT\ll 1$, then configurations with
a fixed product $\lafrac\lr$ will give near-identical dynamics.
For larger learning rates, this is not the
case. For example $\lafrac>1$ decreases stability non-linearly
(Appendix \ref{app:det_dynamics}). However, as we will show in the next section, nontrivial $\lafrac$ can bring net benefits to optimization.

\subsection{Lookahead can optimize faster than SGD}

An immediate question is: are there scenarios where nontrivial \la ($\lafrac\neq 1$) optimizes faster than SGD?
Perhaps surprisingly, the answer is yes. We can prove the following theorem about the inifinite-time
convergence rate of \la induced by the spectral radius of the per-cycle
transition matrix:

\begin{restatable}{theorem}{fastspectral}
\label{thm:la_optimizes_faster_spectral}
Consider the asymptotic infinite-time convergence rate of the optimization dynamics, governed by the spectral
radius $\rho$ of the transition matrix $\tmat$ which controls the single cycle
dynamics via Equation \ref{eq:diloco_pvec}. Let $\TT\geq 2$ be even, and let $\lmat$
be non-degenerate. If the operator $\T_{1}$ controlling the one-step dynamics of
SGD has non-degenerate largest eigenvalue at the optimal learning rate $\lr^{*}$
for SGD ($\lafrac = 1$), then there exist settings of LA with $\lafrac>1$ which
have faster infinite-time convergence rate than SGD with the optimal learning
rate for that spectrum and batch size.
\end{restatable}

The proof can be found in Appendix \ref{app:spectral_proof}. The key is that using
a larger outer learning rate combined with a smaller inner learning rate while preserving their product gives,
on net, a favorable signal-to-noise ratio tradeoff. This effect also seems to hold for
finite-time horizons, see Appendix \ref{app:la_faster_numerics} for details.
Our theoretical results suggest that the benefits of \la might come from a pure optimization advantage, rather than from implicit
regularization, and that training algorithms which can trade off singal-to-noise ratio without adjusting the base learning rate are generally useful.

\subsection{Dynamics of \ladiloco}

We can extend the analysis of \la to \ladiloco --- that is, \ladiloco with SGD
as both the inner and outer optimizer. Given $\R$ workers indexed by $k$, the dynamics is given by:
\begin{equation}
\psb_{t;k} := \th_{t}{\rm~(when~} t//\TT=0)
\end{equation}
\begin{equation}
\psb_{t+1;k}-\psb_{t;k} = -\frac{\lr}{\B} \J^{\tpose}\pmat_{t;k}\z(\psb_{t;k}){\rm~for~}\TT{\rm~steps}
\end{equation}
\begin{equation}
\th_{t+T}-\th_{t} = \lafrac\left(\frac{1}{\R}\sum_{k=1}^{\R}\psb_{t+\TT;k}-\th_{t}\right)
\end{equation}
where the $\pmat_{t;k}$ are all i.i.d. distributed as the $\pmat$ before.
This is similar to the dynamics of \la, except there are $\R$ independent copies of the
parameters that are propagated for $\TT$ steps. The average of these parameters is then
propagated to the outer optimizer. Here $\B$ is the batch size per-worker, so the total batch
size across workers is $\R\B$.

In the linear regression setting, the dynamics can once again be written over function
space and we can derive the high dimensional dynamics of $\pvec$ (Appendix \ref{app:pvec_derivation}):
\begin{equation}
\pvec_{t+\TT}  = \left([(1-\lafrac)\Id+\lafrac(\Id-\lr\lmat)^{\TT}]^{2}+\frac{\lafrac^{2}}{\R}[\T_{\TT}-(\Id-\lr\lmat)^{2\TT}]\right)\pvec_{t}
\label{eq:diloco_pvec}
\end{equation}
This is identical to the dynamics of \la, but with the noise reduced by $\R$. This suggests
that adding more workers with the same $\B$ is always beneficial to
training in this setting.

A more interesting and practically relevant scenario is one where the \emph{total batch size} remains fixed, while the
number of workers is varied. This allows, for example, a comparison between data-parallel
training and \ladiloco style training. If $\R>1$ is optimal here it would suggest that even
when synchronous training is no more costly than asynchronous training, it is beneficial to
decouple the parameters across workers.

To analyze this setting, it is useful to examine the structure of the noise term --- the only place where
\ladiloco departs from \la. 
It is informative to write the full noise term as a power series.
We note that $\T_{\TT}$ can be written as the sum of all possible $\TT$-words made of
products of $\amat$ and $\bmat$. We define the power series $\mmat_{s}$ by $\mmat_{s} \equiv \sum_{\m{W}\in \pi^{\TT}_{s}}\m{W}$,
where $\pi^{\TT}_{s}$ is the set of products of length $\TT$ of $\amat$ and $\bmat$ with
exactly $s$ copies of $\amat$. We note that $\mmat_{0} = (\Id-\lr\lmat)^{2\TT}$.
A simple but tedious calculation shows us that for $s>0$, $\mmat_{s}$ can be written as
$\frac{\lr^{2s}}{\B^{s}}\tl{\mmat}_{s}(\lr)$, where $\tl{\mmat}_{s}$ is independent of $\B$
and $\R$:
\begin{equation}
\frac{\lafrac^{2}}{\R}[\T_{\TT}-(\Id-\lr\lmat)^{2\TT}] = \sum_{s=1}^{\TT}a_{s}\tl{\mmat}_{s}(\lr),~a_{s}\equiv \frac{\lafrac^{2}\lr^{2s}}{\R\B^{s}}
\end{equation}

In terms of the total batch size $\B_{tot} = \B\R$, the coefficients are given by
\begin{equation}
a_{s} = \frac{\R^{s-1}}{\B_{tot}^{s}}\lafrac^{2}\lr^{2s}
\end{equation}
This shows that at the same \emph{total batch size}, using multiple workers $\R>1$ (\ladiloco) generates
\emph{more noise} than single worker \la at matching $\lr$ and $\lafrac$. This suggests that in this setting, using \ladiloco does not improve optimization --- unlike the benefit
gained from switching to \la from SGD. This can be formally proven for near-isotropic data (Appendix \ref{app:slowloco_proof}):
\begin{restatable}{theorem}{slowloco}
Consider the dynamics of Equation \ref{eq:diloco_pvec} for fixed $\lafrac$ and
$\lr$. Then if $\lmat = \lam_{0}\Id$, $\lam_{0}>0$,
for fixed total batch size $\B_{\rm tot}\equiv \B\R$, the eigenvalues of the linear
system strictly increase with $\R$.
Additionally, given a diagonal matrix $\m{\Delta}$, given a pair $\R_{1}<\R_{2}$, there exists an $\epsilon>0$
such that the eiegenvalues for the system with $\R_{2}$ workers has larger eigenvalues than the system for $\R_{1}$ workers for $\lmat = \lam_{0}\Id+\epsilon\m{\Delta}$.
\end{restatable}

Equation \ref{eq:diloco_pvec} and the power series decomposition
can also be used to derive the relationships between hyperparameters in the two
methods. We first consider the simple case where $\lr\lam\TT\ll 1$ and
$\lafrac\lr\lam\TT\ll 1$ for all eigenvalues $\lam$ in $\lmat$.
In this setting the noise is small, and the dynamics is dominated by the deterministic term.
The gradient varies slowly, and the update equation is approximately $\pvec_{t+\TT}\approx (1-\lafrac\lr\lmat)\pvec_{t}$. In this setting, keeping the product $\lafrac\lr$
constant preserves the dynamics, across values of $\lr$ and $\R$.

The more interesting case is where there is non-trivial noise. Consider two training
scenarios: one with \la with learning rate pair $(\lafrac_{\rm loc}, \lr_{\rm loc})$ and
the other with \ladiloco with $\R$ workers and learning rate pair $(\lafrac_{\rm dil}, \lr_{\rm dil})$. If the total batch size $\B_{tot}$ is fixed across the methods, we can compute the
ratio of the power series coefficients $a_{s}$. The $\B_{tot}$ dependence vanishes and we
are left with:
\begin{equation}
\frac{a_{s, {\rm dil}}}{a_{s, {\rm loc}}} = \left(\frac{\lafrac_{\rm dil}}{\lafrac_{\rm loc}}\right)^{2}\R^{s-1}\left(\frac{\lr_{\rm dil}}{\lr_{\rm loc}}\right)^{2s}
\label{eq:coeff_rat_full}
\end{equation}
This ratio can be set to $1$ with the choice
\begin{equation}
\lr_{\rm dil} := \R^{-1/2}\lr_{\rm loc},~\lafrac_{\rm dil}\lr_{\rm dil}  := \lafrac_{\rm loc}\lr_{\rm loc}
\label{eq:r_sqrt_rule}
\end{equation}
This suggests that one way to combat the increase in noise due to additional
workers is to decrease the inner learning rate, but compensate by increasing
the outer learning rate.
In Appendix \ref{app:r_one_half_numerics} we provide numerical evidence for
the linear case demonstrating that this scaling rule makes learning dynamics
similar over a non-trivial range of $\R$ and $\lr$. Because of the higher order
terms the correspondence is not exact and breaks down for large $\R$ and large $\lr$.
We get similar results for a two-layer MLP on MNIST (Appendix \ref{app:nonlinear_scaling}),
suggesting that these effects may survive in the non-linear setting.

In summary, our analysis suggests that at similar total batch sizes, synchronous training is
favored over asynchronous training. However, the outer learning rate in
\ladiloco can partially compensate for the increased noise generation by using
the theoretically motivated scaling rule.

\section{The benefits of dual momentum}

\label{sec:sla}

In practical scenarios, \diloco is used with more complex optimizers in
both the inner and outer loop. In particular, using momentum in both optimizers
is crucial to successful methods. Successful formulations will often use
Nesterov-style momentum in the outer loop \citep{charles2025communication, kallusky2025snoo},
which we also found to be better than EMA style momentum
(Appendix \ref{app:algoperf_experiments}).

In order to improve understanding of what momentum does in these methods, and to
generate hypotheses for why Nesterov-style momentum is preferred,
we analyze ``Super''-Lookahead (\sla) ---
LA with SGD momentum in both the inner and outer optimizers. We focus on the deterministic
setting to obtain interpretable results, but note that existing methods can be used in-principle to extend
the analysis to the stochastic setting
(see \citet{collins2024hitting} for the relevant methodology).

We show that the outer optimizer acts on a normalized, ``effective'' problem, where the outer momentum can
accelerate both small eigenmodes, but also those at the edge of stability, and this effect is amplified
with Nesterov-style momentum.
Our results are compatible with prior work which used worst-case analysis to
analyze the effective condition number of related optimizers \citep{charles2021convergence,charles22iterated},
which also found an ``effective geometry'' that determines the dynamics of a two-phase optimizer,
with improved bounds for Nesterov momentum.


\subsection{Review of single momentum}

\label{sec:single_mom_review}

We review the behavior of
heavy ball and Nesterov momentum in our linear regression setting,
establishing notation and building intuition for later analyses.
EMA style momentum can be written as:
\begin{equation}
\zetvec_{t+1} = (1-\bone)\g_{t}+\bone \zetvec_{t}
\end{equation}
\begin{equation}
\th_{t+1} = \th_{t}-\lr\zetvec_{t+1}
\end{equation}
where $\g$ is the gradient at $\th$. Similar to the SGD, we can write the update equations in function space,
using the residuals $\z$ and the function space momentum $\kvec\equiv \J\zetvec$. The resulting dynamics diagonalize
over the eigenbasis of $\ntk$; for each eigenvalue $\lam$ the projections
$z$ and $k$ onto the eigenmodes evolve as
\begin{equation}
\begin{pmatrix}
z_{t+1}\\
k_{t+1}
\end{pmatrix} = \begin{pmatrix}
1-\lr(1-\bone)\lam & -\lr\bone\\
(1-\bone)\lam & \bone
\end{pmatrix}\begin{pmatrix}
z_{t}\\
k_{t}
\end{pmatrix}
\end{equation}
The convergence rate of the eigenmode is determined by the magnitude of the largest
eigenvalue $\om_{\rm max}$ of the transition matrix, since applying the transition matrix $t$ times results in
a total max eigenvalue of $\om_{\rm max}^{t}$. We define the convergence rate
$\convr\equiv -\log(|\om_{\rm max}|)$.
The eigenvalues $\om_{\pm}$ are given by
\begin{equation}
\om_{\pm}  = \frac{1}{2}\left[(1-\lr(1-\bone)\lam+\bone)
\pm\sqrt{(1-\lr(1-\bone)\lam+\bone)^{2}-4\bone}\right]
\end{equation}
For small learning rates $\lr\lam\ll1-\bone$ the largest eigenvalue is $\om_{+}\approx 1-\lr\lam$, with associated convergence rate $\convr \equiv -\log(1-\lr\lam)\approx \lr\lam$ ---the same as SGD. The dynamics diverges for eigenvalues where
$\lr\lam>2\frac{1+\bone}{1-\bone}$, allowing a maximal learning rate that is
$\frac{1+\bone}{1-\bone}$ larger than GD. The ability to use this
larger learning rate without sacrificing convergence of larger eigenmodes
allows for faster convergence of smaller eigenmodes.

One interesting phenomenon induced by momentum is the existence of complex eigenmodes.
These correspond to oscillatory modes whose convergence rate given by the norm of their eigenvalues $|\om_{\pm}|=\sqrt{\bone}$ --- which is independent of $\lr$ and $\lam$.
If $\bone = 1-\dl$, $\dl\ll1$ the
convergence rate of the modes is $\convr\approx \dl/2$ --- a potentially slow convergence.
These modes occur when $(1-\lr(1-\bone)\lam+\bone)^{2}-4\bone<0$, which for $1-\bone\ll 1$
corresponds to
\begin{equation}
\frac{1-\bone}{4}<\lr\lam<\frac{2(1+\bone)}{1-\bone}
\end{equation}
That is, for $\bone$ close to $1$ there is a range of $\lr\lam$ spanning a factor of
$O((1-\bone)^{-2})$ which converge slowly. This is one of the issues with using $\bone$
too close to $1$; even in this simple problem there is a tradeoff between stability and
convergence rate that can't be mitigated through choice of learning rate.
Similarly, for Nesterov-style momentum we have:
\begin{equation}
\zetvec_{t} = (1-\bone)\g_{t}+\bone \zetvec_{t-1}
\end{equation}
\begin{equation}
\th_{t+1} = \th_{t}-\lr[(1-\bone) \g_{t}+\bone \zetvec_{t}]
\end{equation}
In the eigenspace this nets the linear system
\begin{equation}
\begin{pmatrix}
z_{t+1}\\
k_{t+1}
\end{pmatrix} = \begin{pmatrix}
1-(1-\bone^{2})\lr\lam & -\lr\bone^{2}\\
(1-\bone)\lam & \bone
\end{pmatrix}\begin{pmatrix}
z_{t}\\
k_{t}
\end{pmatrix}
\end{equation}
This form is in units which matches the dynamics of EMA momentum for small eigenvalues
$\lr\lam\ll1$. The traditional form can be recovered with the substitution $\tl{\lr}:= (1-\bone)\lr$.
$\om_{\pm}$ are complex when
\begin{equation}
\frac{(1-\bone)}{(1+\bone)^{2}} < \lr\lam < \frac{1}{1-\bone}
\end{equation}
once again giving a range of $O((1-\bone)^{-2})$ for complex eigenvalues. However, for
Nesterov-style momentum the eigenvalue magnitude and convergence rate depend on $\lr\lam$:
\begin{equation}
|\om_{\pm}| = \sqrt{\bone(1-(1-\bone)\lr\lam)},~
\convr  = -\frac{1}{2}\log(\bone(1-(1-\bone)\lr\lam))
\end{equation}
For $\lr\lam\ll(1-\bone)^{-1}$ with complex eigenvalues we still have the same convergence issue as
before; however, for $\lr\lam\sim (1-\bone)^{-1}$ (near the end of the complex region)
the convergence is much improved. We will see in the
next section that this makes Nesterov momentum uniquely suited to use in the outer optimizer
of \sla.

\subsection{\sla and dual momenta}

We define the general Super Lookahead (\sla) optimizer as:
\begin{equation}
\psb_{t} := \th_{t}{\rm~(when~} t//\TT=0)
\end{equation}
\begin{equation}
\psb_{t+1}-\psb_{t} = InnerOpt(\g(\psb_{t})){\rm~for~}\TT{\rm~steps}
\end{equation}
\begin{equation}
\th_{t+T} = \th_{t}- OuterOpt(\th_{t}-\psb_{t+T})
\end{equation}
where $InnerOpt$ and $OuterOpt$ are gradient based optimizers. The outer optimizer uses
the difference $\th_{t}-\psb_{t+T}$ in place of a gradient.
We consider the case where $InnerOpt$ and $OuterOpt$ are GD with EMA style momentum (full definition in Appendix \ref{app:sla_ema_def}).
We have inner and outer learning rates $\lr$ and $\lafrac$ as before, and now have
inner and outer momentum parameters $\bin$ and $\bout$.
In this setting the function space updates can be written as:
\begin{equation}
\tz_{t} := \z_{t}{\rm~(when~} t//\TT=0)
\label{eq:SLA_start}
\end{equation}
\begin{equation}
\left.
\begin{aligned}
\tz_{t+1} & = (\Id-\lr(1-\bin)\ntk)\tz_{t}-\lr\bin\tk_{t}\\
\tk_{t+1} & = (1-\bin)\ntk\tz_{t}+\bin\tk_{t}
\end{aligned}
\right\} \quad {\rm~for~}\TT{\rm~steps}
\end{equation}
\begin{equation}
\kvec_{t+T} = -(1-\bout)[\tz_{t+T-1}-\z_{t}] +\bout\kvec_{t},~\z_{t+\TT} = \z_{t}-\lafrac\kvec_{t+T}
\label{eq:SLA_end}
\end{equation}
We note that only $\tz_{t}$ is synced between cycles; $\tk$ is only modified via
the inner optimizer.
In this setting, momentum rules can be generally written as
\begin{equation}
\begin{pmatrix}
\tz_{t+\TT} \\
\tk_{t+\TT} \\
\z_{t+\TT} \\
\kvec_{t+\TT}
\end{pmatrix} = 
\Tsyn \Tout \Tin \begin{pmatrix}
\tz_{t} \\
\tk_{t} \\
\z_{t} \\
\kvec_{t}
\end{pmatrix}
\end{equation}
where $\Tin$ represents the $\TT$ steps of the inner optimizer, $\Tout$ represents the outer
optimizer step, and $\Tsyn$ represents the syncing of inner and outer optimizers (in this case,
the syncing of $\tz_{t}$ to $\z_{t}$ at the end of the cycle. By analyzing the eigenvalues of
the product $\Ttot \equiv \Tsyn \Tout \Tin$, we can understand the convergence rates of
SLA as a function of the optimizer parameters. Detailed expressions for the matrices can be
found in Appendix \ref{app:sla_det_dyn_full}.

\subsection{Analysis of a $2$-dimensional reduction}

In order to build intuition for the convergence behavior, it is useful to define a variant
of the optimizer where the eigenvalues of $\Ttot$ are analytically tractable. Consider the
additional rule $\tk_{t}:=0$ at the end of the $\TT$ step cycle. This modifies $\Tsyn$ only.
Working once again in the $\ntk$ eigenbasis, the dynamics per-cycle for each eigendirection
can be described in terms of the projections $\zz$ and $\kk$ alone, since no information
from $\tz$ or $\tk$ is carried over cycle-to-cycle. Then $\Tin$ is given by
\begin{equation}
\Tin = \begin{pmatrix}
\mF & 0\\
0 & \Id
\end{pmatrix},~
\mF = \begin{pmatrix}
1-(1-\bin)\lr\lam & -\eta\bin \\
(1-\bin)\lam & \bin
\end{pmatrix}^{\TT}
\end{equation}
then the update rule in an eigendirection reduces to
\begin{equation}
\begin{pmatrix}
\zz_{t+\TT} \\
\kk_{t+\TT}
\end{pmatrix} = 
\begin{pmatrix}
1-\lafrac(1-\bout)[1-\FF_{11}] & -\lafrac\bout\\
(1-\bout)[1-\FF_{11}] & \bout\\
\end{pmatrix}
\begin{pmatrix}
\zz_{t} \\
\kk_{t}
\end{pmatrix}
\end{equation}
Therefore \sla is equivalent to Momentum-GD on an effective spectrum
$1-\FF_{11}$. This structure holds for any inner optimizer whose steps
can be written as a linear function of $\zz_{t}$ and $\kk_{t}$.

If $1-\FF_{11}\ll 1$, then the corresponding convergence rate per-cycle is
$\convr\approx \lafrac (1-\FF_{11})$. This occurs for small eigenvalues $\lr\lam\TT\ll 1$,
where the \emph{per step} convergence rate
$\convr\equiv -\frac{1}{\TT}\log(|\om_{\rm max}|)$, where $\om$ are the eigenvalues of $\Ttot$
is given by $\convr\approx \lafrac\lr\lam$.
We get acceleration over GD if
$\lafrac\lr>\lr_{\rm GD}$.

We also have $1-\FF_{11}\ll 1$ for slowly converging eigenmodes near the edge of stability.
For example if $\bin = 0$
(no inner momentum) and $\lr\lam = 2-\dl$ for $\dl\ll 1$, then $\FF_{11} \approx 1-\TT\dl$ and
the per-cycle convergence rate becomes $\convr\approx \lafrac\dl$. A similar effect applies to the inner
eigenmodes with imaginary eigenvalues if $\TT(1-\bin)\ll1$.
In these cases, we have acceleration if $\lafrac >1$. Therefore outer momentum can improve convergence of
any slowly converging mode with sufficiently large $\lafrac$.

The downside of outer momentum is that any modes which make significant optimization progress
in the inner loop will be slowed down. For example, consider damped eigenmodes; for $\bout = 1-\dl$,
$\dl\ll1$ the per-step convergence rate will be $\convr = \frac{1}{\TT}(1-\dl/2)$ -- \emph{decreasing} for longer cycles,
since the \emph{per-cycle} convergence rate only depends on $\bout$.
The best-optimizing eigenmodes in the inner optimizer can become the
slowest in the outer optimizer. These results suggest that $\bout$ will likely need to be
even further from $1$ than $\bin$.

These convergence issues can be greatly mitigated by using Nesterov-style momentum.
From the full dynamics in Appendix \ref{app:sla_nest_def}, combined with inner momentum
zeroing out
the $2$-d dynamics becomes:
\begin{equation}
\begin{pmatrix}
\zz_{t+\TT} \\
\kk_{t+\TT}
\end{pmatrix} = 
\begin{pmatrix}
1-\lafrac(1-\bout^{2})(1-\FF_{11}) & -\lafrac\bout^{2}\\
(1-\bout)(1-\FF_{11}) & \bout
\end{pmatrix}
\begin{pmatrix}
\zz_{t} \\
\kk_{t}
\end{pmatrix}
\end{equation}
Once again, the effective eigenvalues are given by $1-\FF_{11}$. This means that in the
damped/complex eigenvalue regime, the per-cycle eigenvalue magnitude is
$\sqrt{\bout(1-(1-\bout)\lafrac(1-\FF_{11}))}$. Any eigenvalues where $\lafrac(1-\FF_{11})\ll (1-\bout)^{-1}$ will still experience slowdown.
Since $1-\FF_{11}\in [0, 1]$ and $(1-\bout)^{-1}>1$, the only hope to combat this is to pick
$\lafrac\geq (1-\bout)^{-1}$ (corresponding to $\lafrac\geq 1$ with traditionally parameterized
Nesterov momentum).

Unlike the inner optimizer, where the learning rate can only be tuned to accelerate a small
range of eigenmodes (those with $\lr\lam\sim(1-\bin)^{-1}$, there is a special value of
$\lafrac$ that can mitigate the damping for many eigenmodes. The choice $\lafrac= (1-\bout)^{-1}$
accelerates all eigenmodes with effective eigenvalues $1-\FF_{11}\approx 1$, or alternatively
$\FF_{11}\ll 1$. There is a large class of eigendirections with this property: namely
eigenmodes which make significant progress on the loss in a single cycle. These modes are
normally damped by the outer optimizer, but with Nesterov momentum and a well-tuned outer
learning rate the damping is mitigated. For example, if $\FF_{11} = ae^{-\tlam\TT}$ as $\TT$
is varied, and
$\lafrac= (1-\bout)^{-1}$, the convergence rate per cycle is $\convr>\TT\tlam$, and per step is
$\convr>\tlam$.

Overall, this analysis shows that applying two momenta at different scales has non-trivial
effects that can't be replicated by a single momentum. With the correct choice of momentum
rule the outer momentum can speed up convergence of eigendirections slowed down by
the inner momentum, while preserving progress on well-converging modes. This effect can't
be captured by simple interventions like adding weight averaging on top of momentum (Appendix \ref{app:weight_ema_vs_sla}) or proposed algorithms like GPA \citep{defazio2025smoothing} which
try to capture some of the effects of \sla (Appendix \ref{app:gpa_vs_sla}).

\subsection{Numerical results}

We analyzed convergence rates for EMA and Nesterov momenta by performing eigendecomposition 
of $\Ttot$ for both the reduced system
analyzed above, as well as the original system
described by Equations \ref{eq:SLA_start} to \ref{eq:SLA_end} where $\tk$ is not reset
to $0$ every cycle. We fixed $\lr = 1$, $\lam = 0.2$, $\bin = 0.9$, $\bout = 0.8$,
and kept the inner optimizer fixed to EMA momentum. This put the inner optimizer in the
imaginary eigenvalue regime.
These values were chosen to illustrate our
general points, which hold across a wide range of parameter values.

The convergence rates of the reduced system and the
original system are similar for larger $\TT$ (Figure \ref{fig:double_mom_conv_rate}).
Resetting $\tk$ to $0$ improves the convergence rate for small $\TT$. If we set
$\lafrac$ to a generic stable value, we see that the convergence rate for Nesterov is
larger than the convergence rate for EMA momentum, by a fixed factor for large $\TT$
(Figure \ref{fig:double_mom_conv_rate}, top). However we see that the per-step convergence
rate decreases as $1/\TT$ because the per-cycle convergence rate is fixed, as predicted
by the theory.

However, if we set $\lafrac$ to the critical value $(1-\bout)^{-1}$, then we see
that the per-step convergence rate of Nesterov momentum goes to a constant as predicted
(Figure \ref{fig:double_mom_conv_rate}, bottom). This highlights the point that the critical
outer learning rate is useful for preserving convergence rates for large $\TT$.

\begin{figure}[h]
\centering
\begin{tabular}{cc}
\includegraphics[width=0.35\linewidth]{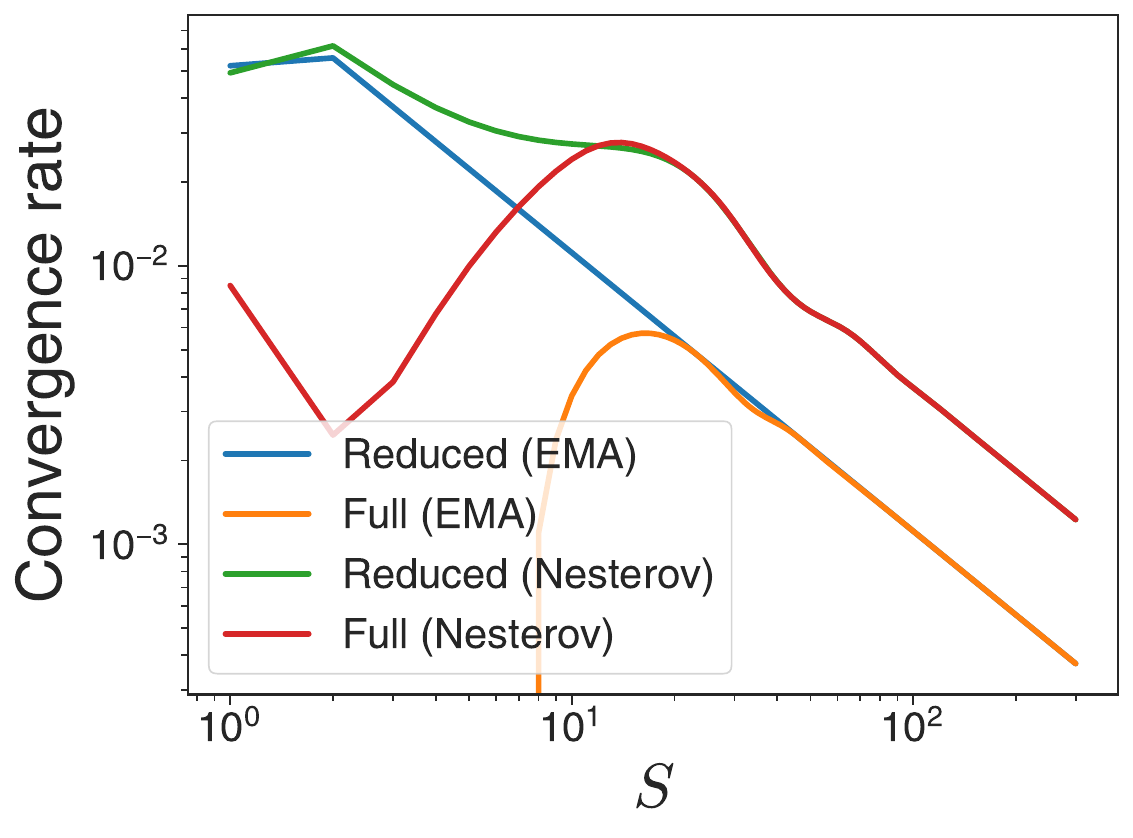}&
\includegraphics[width=0.35\linewidth]{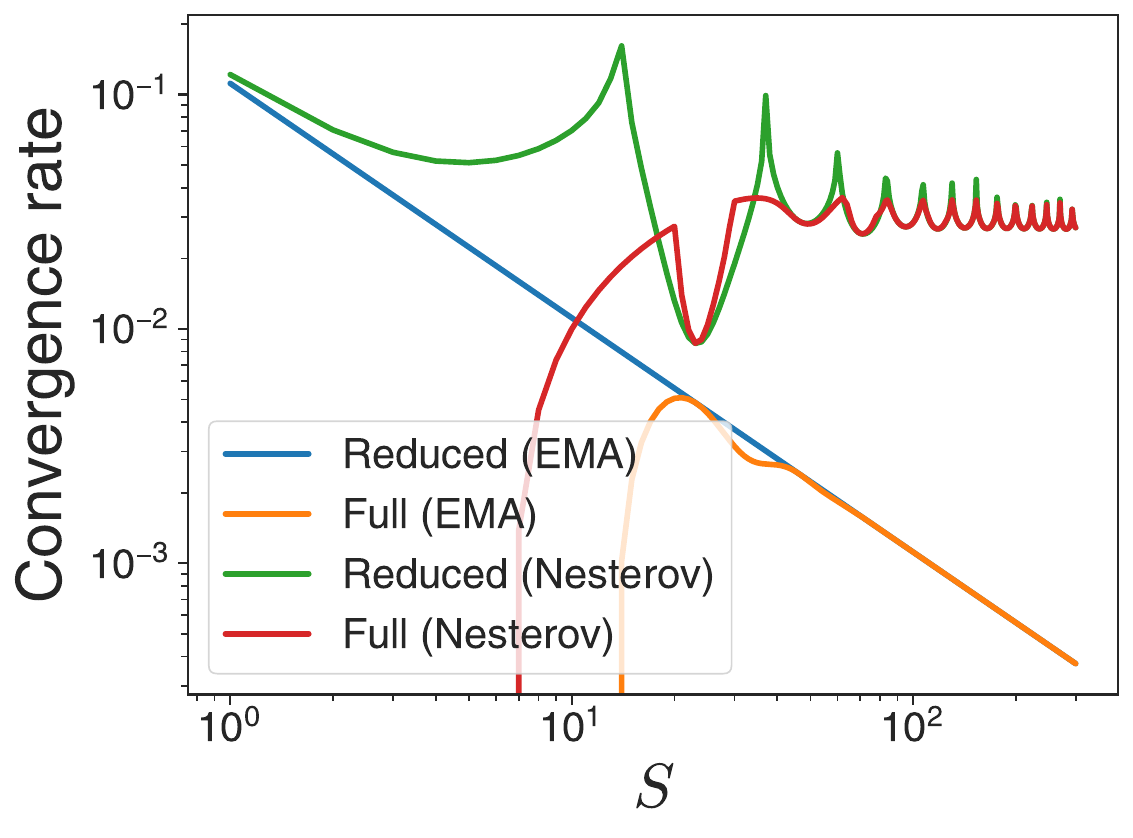}
\end{tabular}
\caption{Eigenmodes damped by inner optimizer show decreasing per-step convergence rate with
momentum-GD in the outer optimizer (left, $\lafrac = 2$). Convergence rate is improved for Nesterov momentum. At the critical $\lafrac = (1-\bout)^{-1}$,
Nesterov convergence rate becomes $\TT$-invariant at large $\TT$ (right). Reduced system with $\tk_{t}:=0$ every cycle has similar
behavior to full dynamics with $\tk_{t}$ preserved for larger $\TT$, but shows faster
convergence/better stability at smaller $\TT$. All experiments have $\lr = 1$, $\lam = 0.2$, $\bin = 0.9$, $\bout = 0.8$.}
\label{fig:double_mom_conv_rate}
\end{figure}

\section{Discussion}

This work represents the first average case analysis of  \ladiloco
in the high dimensional setting, and reveals the complex tradeoff between signal and noise that
controls the convergence rates. Even with a single worker \ladiloco is quantitatively different from SGD, with provably
improved convergence rates in certain settings.
Our analysis suggests that all else equal, synchronous optimizers are better
than asynchronous ones, but it is possible to design good hyperparameter transfer schemes for a modest number of workers. Our analysis of the momentum-equipped \sla variants in the
deterministic setting give new insights into the opportunities for acceleration, and suggest that Nesterov momentum minimizes the downsides of the approach ---
mirroring and perhaps explaining observations in practical settings \citep{charles2025communication, kallusky2025snoo}.

Overall the benefits of two-phase optimizers in this setting come from the fact that the
outer optimizer operates on an ``effective'' problem defined by the action of the inner
optimizer. These benefits are due to new tradeoffs between speed and stability surfaced by the structure of the optimizer, that are easily accessible via hyperparameter tuning.
These effects can't be recreated with
traditional optimizers equipped with weight averaging (Appendix \ref{app:weight_ema_vs_sla}), nor by
algorithms like GPA \citep{defazio2025smoothing} which seek to mimic the benefits of \sla (Appendix \ref{app:gpa_vs_sla}).
In general it seems that there is a benefit to surfacing control over the variance properties of the optimizer without affecting the
``overall'' learning rate.

The main weakness of our analysis is that it assumes the geometry of the loss landscape
(e.g. the Hessian/NTK) is fixed. One next step would be to analyze cases where the
eigenvalues and eigenvectors of the Hessian shift over time, or where the loss landscape is
not well characterized by first and second derivatives on the scales relevant for
the learning dynamics. In particular, it is known that algorithmic choices induce feedback
on the large eigenvalues of the loss landscape \citep{cohen2021gradient, cohen2022adaptive, roulet2024stepping}; understanding how
this phenomenon affects two-phase optimizers is likely crucial to improve these methods
further, and scale them to larger workloads.

Our analysis suggests that the optimization dynamics of these two-phase optimizers
has unique properties not found in most commonly used training algorithms.
Though the periodic structure is an aesthetic departure from typical algorithms, the ability
to explicitly access different time and curvature scales is robustly beneficial.
We speculate that these effects can eventually be bridged with other methods that exploit multiple timescales 
\citep{defazio2024road, ferbach2025dimension}. The insights presented here are just the starting point to
new exciting paradigms in training algorithms.

\section{Acknowledgments}

The author would like to thank Keith Rush and Zachary Charles for inspiring the project,
lending their expertise about
\diloco and other optimizers, and their guidance in preparing the manuscript.


\bibliography{sla_refs}

@inproceedings{charles2021convergence,
  title={Convergence and accuracy trade-offs in federated learning and meta-learning},
  author={Charles, Zachary and Kone{\v{c}}n{\`y}, Jakub},
  booktitle={International Conference on Artificial Intelligence and Statistics},
  pages={2575--2583},
  year={2021},
  organization={PMLR}
}

@article{charles2025communication,
  title={Communication-Efficient Language Model Training Scales Reliably and Robustly: Scaling Laws for DiLoCo},
  author={Charles, Zachary and Teston, Gabriel and Dery, Lucio and Rush, Keith and Fallen, Nova and Garrett, Zachary and Szlam, Arthur and Douillard, Arthur},
  journal={arXiv preprint arXiv:2503.09799},
  year={2025}
}

@article{douillard2023diloco,
  title={Diloco: Distributed low-communication training of language models},
  author={Douillard, Arthur and Feng, Qixuan and Rusu, Andrei A and Chhaparia, Rachita and Donchev, Yani and Kuncoro, Adhiguna and Ranzato, Marc'Aurelio and Szlam, Arthur and Shen, Jiajun},
  journal={arXiv preprint arXiv:2311.08105},
  year={2023}
}

@article{zhang2019lookahead,
  title={Lookahead optimizer: k steps forward, 1 step back},
  author={Zhang, Michael and Lucas, James and Ba, Jimmy and Hinton, Geoffrey E},
  journal={Advances in neural information processing systems},
  volume={32},
  year={2019}
}

@article{collins2024hitting,
  title={Hitting the high-dimensional notes: An ode for sgd learning dynamics on glms and multi-index models},
  author={Collins-Woodfin, Elizabeth and Paquette, Courtney and Paquette, Elliot and Seroussi, Inbar},
  journal={Information and Inference: A Journal of the IMA},
  volume={13},
  number={4},
  pages={iaae028},
  year={2024},
  publisher={Oxford University Press}
}

@article{lee2022trajectory,
  title={Trajectory of mini-batch momentum: batch size saturation and convergence in high dimensions},
  author={Lee, Kiwon and Cheng, Andrew and Paquette, Elliot and Paquette, Courtney},
  journal={Advances in Neural Information Processing Systems},
  volume={35},
  pages={36944--36957},
  year={2022}
}

@article{kallusky2025snoo,
  title={SNOO: Step-K Nesterov Outer Optimizer-The Surprising Effectiveness of Nesterov Momentum Applied to Pseudo-Gradients},
  author={Kallusky, Dominik and Rao, Vinay and Nandavanam, Vishal and Shi, Hao-Jun Michael},
  journal={arXiv preprint arXiv:2510.15830},
  year={2025}
}

@inproceedings{reddi2021adaptive,
  title={Adaptive Federated Optimization},
  author={Sashank J. Reddi and Zachary Charles and Manzil Zaheer and Zachary Garrett and Keith Rush and Jakub Kone{\v{c}}n{\'y} and Sanjiv Kumar and Hugh Brendan McMahan},
  booktitle={International Conference on Learning Representations},
  year={2021},
  url={https://openreview.net/forum?id=LkFG3lB13U5}
}

@inproceedings{mcmahan2017communication,
  title={Communication-efficient learning of deep networks from decentralized data},
  author={McMahan, Brendan and Moore, Eider and Ramage, Daniel and Hampson, Seth and y Arcas, Blaise Aguera},
  booktitle={Artificial intelligence and statistics},
  pages={1273--1282},
  year={2017},
  organization={PMLR}
}

@article{khaled2025understanding,
  title={Understanding outer optimizers in local sgd: Learning rates, momentum, and acceleration},
  author={Khaled, Ahmed and Kale, Satyen and Douillard, Arthur and Jin, Chi and Fergus, Rob and Zaheer, Manzil},
  journal={arXiv preprint arXiv:2509.10439},
  year={2025}
}

@article{agarwala2024high,
  title={High dimensional analysis reveals conservative sharpening and a stochastic edge of stability},
  author={Agarwala, Atish and Pennington, Jeffrey},
  journal={arXiv preprint arXiv:2404.19261},
  year={2024}
}

@article{roulet2024stepping,
  title={Stepping on the edge: Curvature aware learning rate tuners},
  author={Roulet, Vincent and Agarwala, Atish and Grill, Jean-Bastien and Swirszcz, Grzegorz and Blondel, Mathieu and Pedregosa, Fabian},
  journal={Advances in Neural Information Processing Systems},
  volume={37},
  pages={47708--47740},
  year={2024}
}

@article{cohen2021gradient,
  title={Gradient descent on neural networks typically occurs at the edge of stability},
  author={Cohen, Jeremy M and Kaur, Simran and Li, Yuanzhi and Kolter, J Zico and Talwalkar, Ameet},
  journal={arXiv preprint arXiv:2103.00065},
  year={2021}
}

@article{cohen2022adaptive,
  title={Adaptive gradient methods at the edge of stability},
  author={Cohen, Jeremy M and Ghorbani, Behrooz and Krishnan, Shankar and Agarwal, Naman and Medapati, Sourabh and Badura, Michal and Suo, Daniel and Cardoze, David and Nado, Zachary and Dahl, George E and Gilmer, Justin},
  journal={arXiv preprint arXiv:2207.14484},
  year={2022}
}

@article{ferbach2025dimension,
  title={Dimension-adapted Momentum Outscales SGD},
  author={Ferbach, Damien and Everett, Katie and Gidel, Gauthier and Paquette, Elliot and Paquette, Courtney},
  journal={arXiv preprint arXiv:2505.16098},
  year={2025}
}

@article{defazio2024road,
  title={The road less scheduled},
  author={Defazio, Aaron and Yang, Xingyu and Mehta, Harsh and Mishchenko, Konstantin and Khaled, Ahmed and Cutkosky, Ashok},
  journal={Advances in Neural Information Processing Systems},
  volume={37},
  pages={9974--10007},
  year={2024}
}

@inproceedings{malinovskiy2020local,
  title={From local SGD to local fixed-point methods for federated learning},
  author={Malinovskiy, Grigory and Kovalev, Dmitry and Gasanov, Elnur and Condat, Laurent and Richtarik, Peter},
  booktitle={International Conference on Machine Learning},
  pages={6692--6701},
  year={2020},
  organization={PMLR}
}

@InProceedings{charles22iterated,
  title = 	 {Iterated Vector Fields and Conservatism, with Applications to Federated Learning},
  author =       {Charles, Zachary and Rush, Keith},
  booktitle = 	 {Proceedings of The 33rd International Conference on Algorithmic Learning Theory},
  pages = 	 {130--147},
  year = 	 {2022},
  editor = 	 {Dasgupta, Sanjoy and Haghtalab, Nika},
  volume = 	 {167},
  series = 	 {Proceedings of Machine Learning Research},
  month = 	 {29 Mar--01 Apr},
  publisher =    {PMLR},
  url = 	 {https://proceedings.mlr.press/v167/charles22a.html},
}

@inproceedings{
douillard2025streaming,
title={Streaming DiLoCo with overlapping communication},
author={Arthur Douillard and Yani Donchev and J Keith Rush and Satyen Kale and Zachary Charles and Gabriel Teston and Zachary Garrett and Jiajun Shen and Ross McIlroy and David Lacey and Alexandre Rame and Arthur Szlam and MarcAurelio Ranzato and Paul R Barham},
booktitle={Second Conference on Language Modeling},
year={2025},
url={https://openreview.net/forum?id=yYk3zK0X6Q}
}

@misc{jaghouar2024opendilocoopensourceframeworkglobally,
      title={OpenDiLoCo: An Open-Source Framework for Globally Distributed Low-Communication Training}, 
      author={Sami Jaghouar and Jack Min Ong and Johannes Hagemann},
      year={2024},
      eprint={2407.07852},
      archivePrefix={arXiv},
      primaryClass={cs.LG},
      url={https://arxiv.org/abs/2407.07852}, 
}

@article{defazio2025smoothing,
  title={Smoothing DiLoCo with Primal Averaging for Faster Training of LLMs},
  author={Defazio, Aaron and Mishchenko, Konstantin and Raman, Parameswaran and Shi, Hao-Jun Michael and Xiao, Lin},
  journal={arXiv preprint arXiv:2512.17131},
  year={2025}
}

@Misc{Dahl2023AlgoPerf,
  title         = {{Benchmarking Neural Network Training Algorithms}},
  author        = {Dahl, George E. and Schneider, Frank and Nado, Zachary and Agarwal, Naman and Sastry, Chandramouli Shama and Hennig, Philipp and Medapati, Sourabh and Eschenhagen, Runa and Kasimbeg, Priya and Suo, Daniel and Bae, Juhan and Gilmer, Justin and Peirson, Abel L. and Khan, Bilal and Anil, Rohan and Rabbat, Mike and Krishnan, Shankar and Snider, Daniel and Amid, Ehsan and Chen, Kongtao and Maddison, Chris J. and Vasudev, Rakshith and Badura, Michal and Garg, Ankush and Mattson, Peter},
  year          = {2023},
  archiveprefix = {arXiv},
  eprint        = {2306.07179},
}

@inproceedings{Kasimbeg2025AlgoPerfResults,
title           = {Accelerating neural network training: An analysis of the {AlgoPerf} competition},
author          = {Kasimbeg, Priya and Schneider, Frank and Eschenhagen, Runa and Bae, Juhan and Sastry, Chandramouli Shama and Saroufim, Mark and Boyuan, Feng and Wright, Less and Yang, Edward Z. and Nado, Zachary and Medapati, Sourabh and Hennig, Philipp and Rabbat, Michael and Dahl, George E.},
booktitle       = {The Thirteenth International Conference on Learning Representations},
year            = {2025},
url             = {https://openreview.net/forum?id=CtM5xjRSfm}
}
\bibliographystyle{unsrtnat}

\newpage
\appendix
\onecolumn

\section{High dimensional dynamics of \la and \ladiloco}

\label{app:highd_dynamics}

\subsection{Derivation of the dynamics of $\pvec$}

\label{app:pvec_derivation}

We derive the dynamics in the case of \ladiloco here; \la is then the special case where
$\R = 1$. We have the dynamics:
\begin{equation}
\tz_{t;k} := \z_{t}{\rm~(when~} t//\TT=0)
\end{equation}
\begin{equation}
\tz_{t+1;k}-\tz_{t;k} = -\lr\frac{\D}{\B}\ntk\pmat_{t;k}\tz_{t;k}{\rm~for~}\TT{\rm~steps}
\end{equation}
\begin{equation}
\z_{t+\TT}-\z_{t} = \frac{\lafrac}{\R}\sum_{k=1}^{\R}(\tz_{t+\TT;k}-\z_{t})
\end{equation}
where the $\pmat_{t;k}$ are i.i.d. projection matrices with exactly $\B$ $1$s on the diagonal and $0$ everywhere else, for randomly chosen indices.
The dynamics of the first moment of $\z$ is given by:
\begin{equation}
\expect[\z_{t+\TT}|\z_{t}] = \left[(1-\lafrac)+\lafrac(1-\lr\ntk)\right]\z_{t}
\end{equation}
The second moment evolves as
\begin{equation}
\expect[\z_{t+\TT}\z_{t+\TT}^{\tpose}|\z_{t}] = \expect\left[\left((1-\lafrac)\z_{t}+\frac{\lafrac}{\R}\sum_{k=1}^{\R}\tz_{t+\TT;k}\right)\left((1-\lafrac)\z_{t}+\frac{\lafrac}{\R}\sum_{k=1}^{\R}\tz_{t+\TT;k}\right)^{\tpose}\right]
\end{equation}
Simplifying, we have:
\begin{equation}
\begin{split}
\expect[\z_{t+\TT}\z_{t+\TT}^{\tpose}|\z_{t}] & = (1-\lafrac)^{2}\z_{t}\z_{t}^{\tpose}+\frac{\lafrac(1-\lafrac)}{\R}\left(\sum_{k=1}^{\R}\expect[\tz_{t+\TT;k}]\z_{t}^{\tpose}+\z_{t}\expect[\tz_{t+\TT;k}^{\tpose}]\right) \\
&+\frac{\lafrac^{2}}{\R^{2}}\sum_{j=1}^{\R}\sum_{k=1}^{\R}\expect[\tz_{t+\TT;j}\tz_{t+\TT;k}^{\tpose}]
\end{split}
\end{equation}
Note that $\tz_{t+\TT;j}$ and $\tz_{t+\TT;k}$ are independent for $j\neq k$. We also note that the first and second moments of $\tz_{t+\TT;k}$ given $\z_{t}$ are identical to those
of $\z_{t+\TT}$ from SGD.  We define:
\begin{equation}
\mum_{t,s}\equiv \expect_{\rm SGD}[\z_{t+s}|\z_{t}],~\covm_{t,s}\equiv \expect_{\rm SGD}[\z_{t+s}\z_{t+s}^{\tpose}|\z_{t}]
\end{equation}
Then we have:
\begin{equation}
\expect[\z_{t+\TT}\z_{t+\TT}^{\tpose}|\z_{t}]  = (1-\lafrac)^{2}\z_{t}\z_{t}^{\tpose}+\lafrac(1-\lafrac)\left(\mum_{t,\TT}\z_{t}^{\tpose}+\z_{t}\mum_{t,\TT}^{\tpose}\right)+\lafrac^{2}\left(1-\frac{1}{\R}\right)\mum_{t,\TT}\mum_{t,\TT}^{\tpose}+\frac{1}{\R}\covm_{t,\TT}
\end{equation}
Grouping terms by $\R$ we have:
\begin{equation}
\expect[\z_{t+\TT}\z_{t+\TT}^{\tpose}|\z_{t}]  = \left[(1-\lafrac)\z_{t}+\lafrac\mum_{t,\TT}\right]\left[(1-\lafrac)\z_{t}+\lafrac\mum_{t,\TT}\right]^{\tpose} +\frac{\lafrac^{2}}{\R}\left(\covm_{t,\TT}-\mum_{t,\TT}\mum_{t,\TT}^{\tpose}\right)
\end{equation}

To obtain the dynamics of $\pvec_{t}$, we use the definition $\pvec_{t} \equiv \lmat^{+}\diag(\V^{\tpose}\expect_{\pmat}[\z_{t}\z_{t}^{\tpose}]\V)$ combined with the
previous results from SGD. We have:
\begin{equation}
\mum_{t,\TT} = (\Id-\lr\ntk)^{\TT}\z_{t}, \V^{\tpose}\mum_{t,\TT}\mum_{t,\TT}^{\tpose}\V = (\Id-\lr\lmat)^{\TT}\V^{\tpose}\z_{t}\z_{t}^{\tpose}\V(\Id-\lr\lmat)^{\TT}
\end{equation}
which allows us to simplify the first two terms as
\begin{equation}
\lmat^{+}\diag(\V^{\tpose}\left[(1-\lafrac)\z_{t}+\lafrac\mum_{t,\TT}\right]\left[(1-\lafrac)\z_{t}+\lafrac\mum_{t,\TT}\right]^{\tpose}\V) = [(1-\lafrac)\Id+\lafrac(\Id-\lr\lmat)^{\TT}]^{2}\pvec_{t}
\end{equation}
In high dimensions, the results on SGD show that $\covm_{t,\TT}$ transforms to $(\amat+\bmat)^{\TT}\pvec_{t}$, where
\begin{equation}
\amat  \equiv(\Id-\lr\lmat)^2,~
~\bmat\equiv \left(\frac{1}{\B}-\frac{1}{\D}\right)\lr^2\lmat\m{1}\m{1}^{\tpose}\lmat.
\end{equation}
This gives us the final equations for $\pvec_{t}$ for \ladiloco:
\begin{equation}
\pvec_{t+\TT}  = \left([(1-\lafrac)\Id+\lafrac(\Id-\lr\lmat)^{\TT}]^{2}+\frac{\lafrac^{2}}{\R}[\T_{\TT}-(\Id-\lr\lmat)^{2\TT}]\right)\pvec_{t}
\end{equation}
where
\begin{equation}
\T_{s} \equiv (\amat+\bmat)^{s} = \left[(\Id-\lr\lmat)^2+\left(\frac{1}{\B}-\frac{1}{\D}\right)\lr^2\lmat\m{1}\m{1}^{\tpose}\lmat\right]^{s}
\end{equation}

\subsection{Deterministic dynamics}

\label{app:det_dynamics}

In the deterministic setting (full batch training), the update equation is:
\begin{equation}
\pvec_{t+1} = [(1-\lafrac)\Id+\lafrac(\Id-\lr\lmat)^{\TT}]^{2}\pvec_{t}
\end{equation}
This diagonalizes in the eigenbasis. To understand convergence it is sufficient to analyze the $1$-dimensional system:
\begin{equation}
p_{t+1} = [(1-\lafrac)+\lafrac(1-\lr\lam)^{\TT}]^{2} p_{t}
\end{equation}

The dynamics is stable/convergent if and only if:
\begin{equation}
[(1-\lafrac)+\lafrac(1-\lr\lam)^{\TT}]^{2}< 1
\end{equation}
Let $a = ((1-\lr\lam)^{\TT})-1$. Then we have:
\begin{equation}
(1+a\lafrac)^{2}< 1\to -\frac{2}{\lafrac}< a< 0
\end{equation}
This gives us
\begin{equation}
 1-\frac{2}{\lafrac}< (1-\lr\lam)^{\TT} < 1
\end{equation}
For $\lafrac = 1$, we recover the familiar $\lr\lam< 2$ to prevent runaway growth. For even $\TT$, for $\lafrac <2$ the condition is the same; after that the critical learning rate
decreases with $\lafrac$. For odd $\TT$, the condition is the same for all $\lafrac<1$. For $\lafrac>1$ gain the stable regime decreases. There is always some stable regime for all
$\lafrac$.

This analysis suggests that for $\lafrac<1$, the stability of the inner SGD sets the stability of the overall dynamics. For larger $\lafrac$, the training becomes harder to
stabilize (requiring a smaller inner learning rate $\lr$), but there is always a convergent
regime.

\subsection{Convergence proof in the spectral norm}

\label{app:spectral_proof}

\fastspectral*

\begin{proof}
Recall that the transition matrix $\tmat$ can be written as
\begin{equation}
\tmat(\lafrac, \lr) = \left[(1-\lafrac)\m{I}+\lafrac(\m{I}-\lr\lmat)^\TT\right]^2 + \lafrac^2\left[\T_\TT(\lr) - (\m{I}-\lr\lmat)^{2\TT}\right]
\end{equation}
where $\T_\TT(\lr) = (\amat(\lr) + \bmat(\lr))^\TT$, $\amat(\lr) = (\m{I}-\lr\lmat)^2$, and $\bmat(\lr) = \gamma \lr^2 \lmat\m{1}\m{1}^\tpose\lmat$ for $\gamma = 1/\B-1/\D$.
Let $\rho(\lafrac, \lr)$ be the corresponding spectral radius.

Consider varying $\lafrac$ and $\lr$ while keeping their product constant. If we fix $\lafrac\lr = c$ and vary $\lafrac$, we have
$\rho(\lafrac, \lr(\lafrac)) = \rho(\lafrac, c/\lafrac)$. This gives us the derivative:
\begin{equation}
\frac{d\rho}{d\lafrac} = \frac{\partial\rho}{\partial\lafrac}-\frac{\partial\rho}{\partial\lr}\frac{c}{\lafrac^{2}}
\end{equation}
At the SGD optimal learning rate $c = \lr^*$, with $\lafrac = 1$, 
$\frac{\partial\rho}{\partial\lr} = 0$. This gives us:
\begin{equation}
\frac{d\rho}{d\lafrac} = \frac{\partial\rho}{\partial\lafrac}
\end{equation}
Therefore, if we can show that $\frac{d\rho}{d\lafrac}\neq 0$ for a particular $(\B, \D)$ combination, there exists a setting of \la
which has a better infinite-time convergence rate than SGD.

Since by assumption $\tmat$ has a unique largest eigenvalue, we can write:
\begin{equation}
\rho(1, \eta^{*}) = \u^{\tpose}\tmat\u
\end{equation}
where $\u$ is the eigenvector corresponding to the largest eigenvalue. This means that the derivative at $\lafrac = 1$ is given by:
\begin{equation}
\frac{d\rho}{d\lafrac} = \u^{\tpose}\left.\frac{d\tmat}{d\lafrac}\right|_{\lafrac=1}\u = 2\u^{\tpose}[\T_\TT - (1-\lr\lmat)^{\TT}]\u
\end{equation}
which gives us
\begin{equation}
\frac{1}{2}\frac{d\rho}{d\lafrac} = \lam_{1}^{\TT} - \u^{\tpose}(1-\lr\lmat)^{\TT}\u
\end{equation}
where $\lam_{1}$ is the largest eigenvalue of the one step operator $\T_{1}$, and is given by
\begin{equation}
\lam_{1} = \u^{\tpose}\T_{1}\u = \u^{\tpose}[(1-\lr\lmat)^{2}+\gamma\lr^{2}\lmat\m{1}\m{1}^{\tpose}\lmat]\u
\end{equation}
for $\gamma = \B^{-1}-\D^{-1}$. If we define
\begin{equation}
a\equiv\u^{\tpose}\lmat\u,~b\equiv \u^{\tpose}\lmat^{2}\u,~c\equiv \u^{\tpose}\lmat\m{1}\m{1}^{\tpose}\lmat\u
\end{equation}
then we have:
\begin{equation}
\lam_{1} = 1-2\lr a+\lr^{2} b+\gamma\lr^{2} c
\end{equation}
Since $\rho = \lam_{1}^{\TT}$ when $\lafrac = 1$ (SGD), at the optimal learning rate $\lr^{*}$ we have
\begin{equation}
\left.\frac{d\lam_{1}}{d\lr}\right|_{\lr=\lr^*} = 0 = -2 a+2\lr^{*} b+\gamma\lr^{*} c
\end{equation}
Substituting, at the optimal learning rate we have:
\begin{equation}
\lam_{1} = 1-\lr^{*}a
\end{equation}

Substituting back into the formula for $d\rho/d\lafrac$, we have:
\begin{equation}
\frac{1}{2}\frac{d\rho}{d\lafrac} = (1-\lr^{*}a)^{\TT} - \u^{\tpose}(1-\lr^{*}\lmat)^{\TT}\u = (1-\lr^{*}\u^{\tpose}\lmat\u)^{\TT} - \u^{\tpose}(1-\lr^{*}\lmat)^{\TT}\u
\end{equation}
This expression is a function of moments of $\lmat$. Since $\TT$ is even, $f(\lam) = (1-\lr^{*}\lam)^{\TT}$ is convex. Therefore we can apply
Jensen's inequality
\begin{equation}
(1-\lr^{*}\u^{\tpose}\lmat\u)^{\TT} < u^{\tpose}(1-\lr^{*}\lmat)^{\TT}\u
\end{equation}
The inequality is strict since $\lmat$ is non-degenerate by assumption. This gives us
\begin{equation}
\frac{d\rho}{d\lafrac} <0
\end{equation}
Therefore, $\rho$ can be reduced by increasing $\lafrac$ while leaving the product of inner and outer $\lr$ fixed;
there exist settings with  which converge faster at infinite time than SGD with the optimal learning rate.
\end{proof}

\subsection{\ladiloco slows convergence rates with isotropic data and fixed total batch size}

\label{app:slowloco_proof}

Here we prove the following theorem:
\slowloco*

\begin{proof}
From Equation \ref{eq:diloco_pvec}, the dynamics of $\pvec$ can be written as:
\begin{equation}
\pvec_{t+1} = \m{F}(\lr,\lafrac,\R,\B_{\rm tot}, \lmat)\pvec_{t}
\end{equation}
Our goal will be to show that $\m{F}$ has eigenvalues that are strictly increasing in $\R$ when
$\lr$, $\lafrac$, and $\B_{\rm tot}$ are fixed. We can explicitly parametrize $\m{F}$ by $\lmat$ and write:
\begin{equation}
\m{F}(\lr,\lafrac,\R,\B_{\rm tot}, \lmat) = \m{Q}(\lr,\lafrac, \lmat)+\sum_{s=1}^{\TT}a_{s}(\lr,\lafrac, \R,\B_{\rm tot})\tl{\mmat}_{s}(\lr, \lmat)
\end{equation}
where $\m{Q}$ and $\tl{\mmat}_{s}$ are independent of $\R$ and $\B_{\rm tot}$. 

First, consider the isotropic base case where $\epsilon = 0$ and $\lmat = \lam_{0}\Id$. By definition, each $\tl{\mmat}_{s}(\lr, \lam_{0}\Id)$ can be written as a sum of alternating products
of matrices of the forms $\lmat$ and $\lmat\m{1}\m{1}^{\tpose}\lmat$. Since the base data is isotropic, $\lmat$ is proportional to the identity
and commutes with $\lmat\m{1}\m{1}^{\tpose}\lmat$. Therefore the $\tl{\mmat}_{s}(\lr, \lam_{0}\Id)$ are sums of products of commuting PSD matrices, which are PSD, so $\tl{\mmat}_{s}(\lr, \lam_{0}\Id)$ are also PSD.

Recall that in this setting, the coefficients are:
\begin{equation}
a_{s}(\lr,\lafrac, \R,\B_{\rm tot}) = \frac{\R^{s-1}}{\B_{\rm tot}^{s}}\lafrac^{2}\lr^{2s}
\end{equation}
Since the power series starts at $s = 1$, for any $0<\R_{1}<\R_{2}$, we have $a_{s}(\R_{2}) > a_{s}(\R_{1})$. 
This means that in the isotropic case, the difference between the operators is:
\begin{equation}
\m{D}_{0} \equiv \m{F}(\R_{2}, \lam_{0}\Id)-\m{F}(\R_{1}, \lam_{0}\Id) = \sum_{s=1}^{\TT}c_{s}\tl{\mmat}_{s}(\lr, \lam_{0}\Id)
\end{equation}
for coefficients $c_{s} = a_{s}(\R_{2}) - a_{s}(\R_{1}) > 0$. Therefore, $\m{D}_{0}$ is a strictly positive definite matrix on the relevant subspace, with some minimum eigenvalue be $\delta > 0$.
By Weyl's Inequality, the rank-ordered eigenvalues of $\m{F}(\R_{2}, \lam_{0}\Id)$ are strictly greater than those of $\m{F}(\R_{1}, \lam_{0}\Id)$.

Now, consider the perturbed spectrum $\lmat = \lam_{0}\Id + \epsilon\m{\Delta}$. We can Taylor expand the matrix-valued function $\m{F}$ around $\epsilon = 0$:
\begin{equation}
\m{F}(\R, \lam_{0}\Id + \epsilon\m{\Delta}) = \m{F}(\R, \lam_{0}\Id) + \epsilon \m{F}_{1}(\R, \m{\Delta}) + \mathcal{O}(\epsilon^2)
\end{equation}
where $\m{F}_{1}$ is the first-order derivative operator with respect to the spectrum.

Taking the difference between $\R_2$ and $\R_1$ for the perturbed system yields:
\begin{equation}
\m{F}(\R_{2}, \lmat) - \m{F}(\R_{1}, \lmat) = \m{D}_{0} + \m{B}
\end{equation}
where $\m{B} = \epsilon [\m{F}_{1}(\R_{2}, \m{\Delta}) - \m{F}_{1}(\R_{1}, \m{\Delta})] + \mathcal{O}(\epsilon^2)$.

By Weyl's Inequality, the minimum eigenvalue of the perturbed difference is bounded by:
\begin{equation}
\lam_{\min}(\m{D}_{0} + \m{B}) \ge \lam_{\min}(\m{D}_{0}) + \lam_{\min}(\m{B})
\end{equation}
Since $\lam_{\min}(\m{D}_{0}) = \delta > 0$ and the magnitude of $\lam_{\min}(\m{B})$ scales linearly with $\epsilon$, we can bound the spectrum as:
\begin{equation}
\lam_{\min}(\m{D}_{0} + \m{B}) \ge \delta - \epsilon \| \m{F}_{1}(\R_{2}, \m{\Delta}) - \m{F}_{1}(\R_{1}, \m{\Delta}) \|_{2} - \mathcal{O}(\epsilon^2)
\end{equation}
As long as the anisotropy $\epsilon$ is sufficiently small---specifically $\epsilon \ll \frac{\delta}{\| \m{F}_{1}(\R_{2}, \m{\Delta}) - \m{F}_{1}(\R_{1}, \m{\Delta}) \|_{2}}$---the lower bound remains strictly positive. 

Thus, $\m{F}(\R_{2}, \lmat) - \m{F}(\R_{1}, \lmat)$ remains positive definite, and the eigenvalues of$\m{F}(\R_{2}, \lmat)$ are greater than
those of $\m{F}(\R_{1}, \lmat)$ for appropriately bounded $\epsilon$
\end{proof}

\subsection{$\R^{-1/2}$ scaling heuristic}

\label{app:r_one_half_numerics}

To further understand the $\R^{-1/2}$ scaling heuristic, we trained the high-dimensional
linear regression model using \ladiloco with various $\R$ and $\lr$.
We compared the scaling suggestion in Equation \ref{eq:r_sqrt_rule} to the baseline of not
scaling $\lr$ at all with $\R$. The resulting curves show that using the same $(\lafrac, \lr)$ pairs across $\R$ causes discrepancies between the trajectories and larger $\R$ values
diverge at relatively low values of $\lr$ (Figure \ref{fig:r_scaling_main}, top). In contrast, the heuristic $R^{-1/2}$ scaling
(with counterscaling of $\lafrac$) keeps learning trajectories approximately collapsed at
small learning rates and keeps convergence at larger learning rates (Figure
\ref{fig:r_scaling_main}, bottom).
In all cases, we found that the theoretical high-dimensional predictions (solid lines)
well-matched the theoretical predictions (dashed lines) at the chosen $\D$ and $\B_{tot}$

\begin{figure}[h]
\centering
\begin{tabular}{cc}
\includegraphics[width=0.3\linewidth]{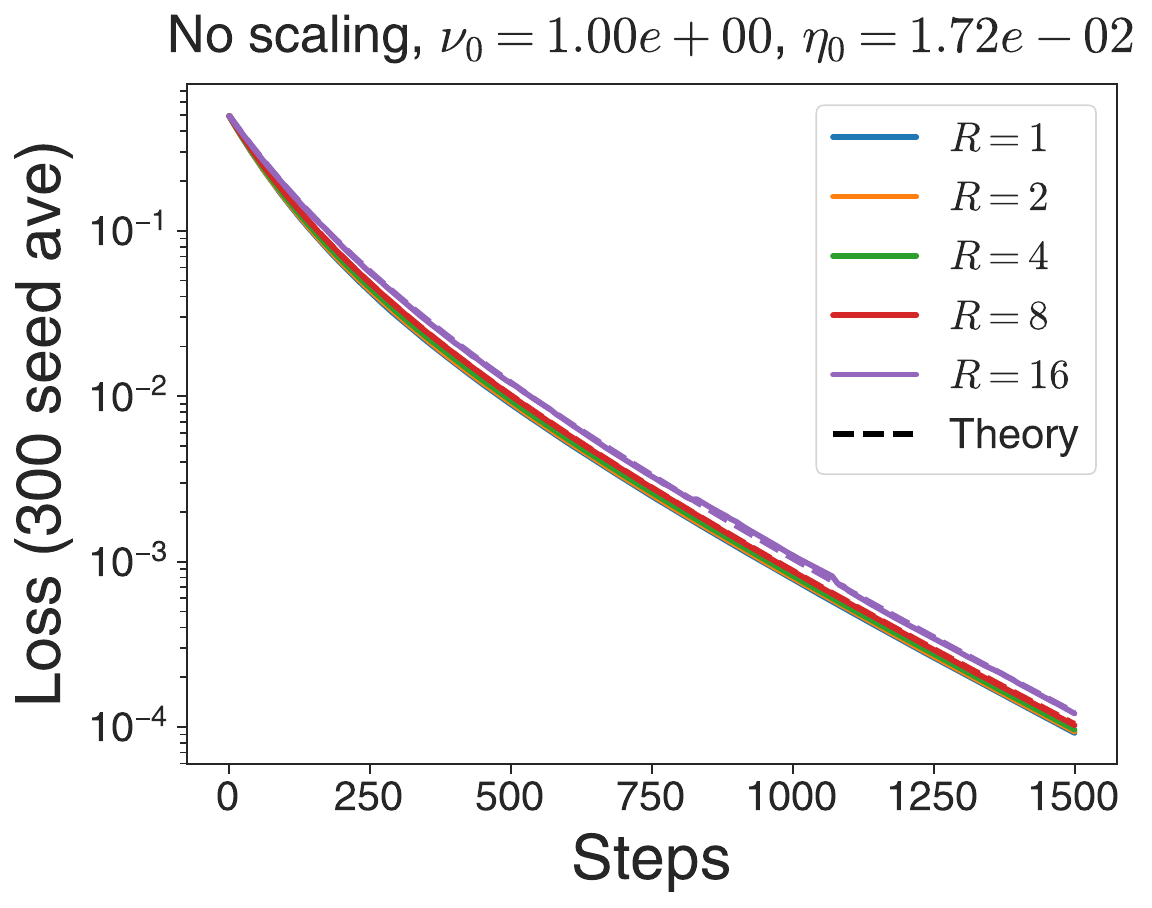} &
\includegraphics[width=0.3\linewidth]{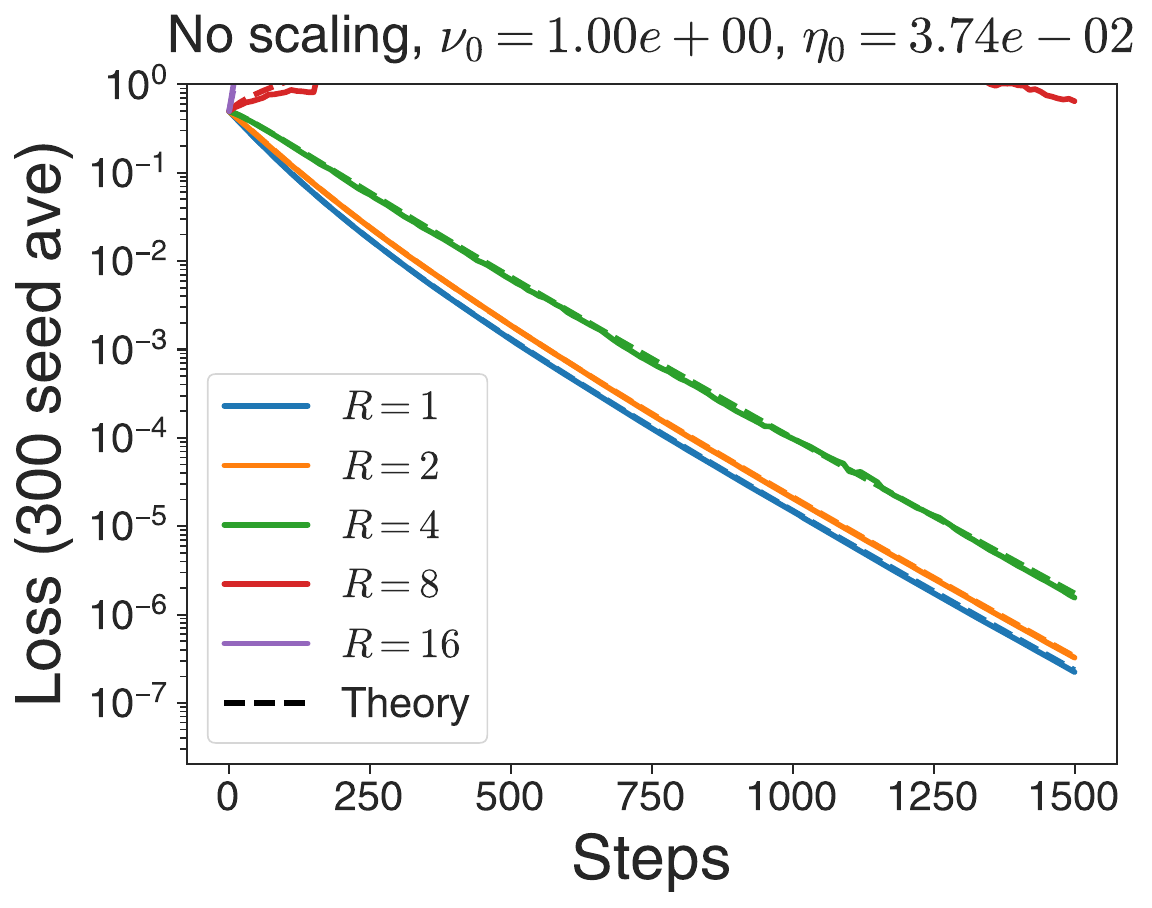}\\
\includegraphics[width=0.3\linewidth]{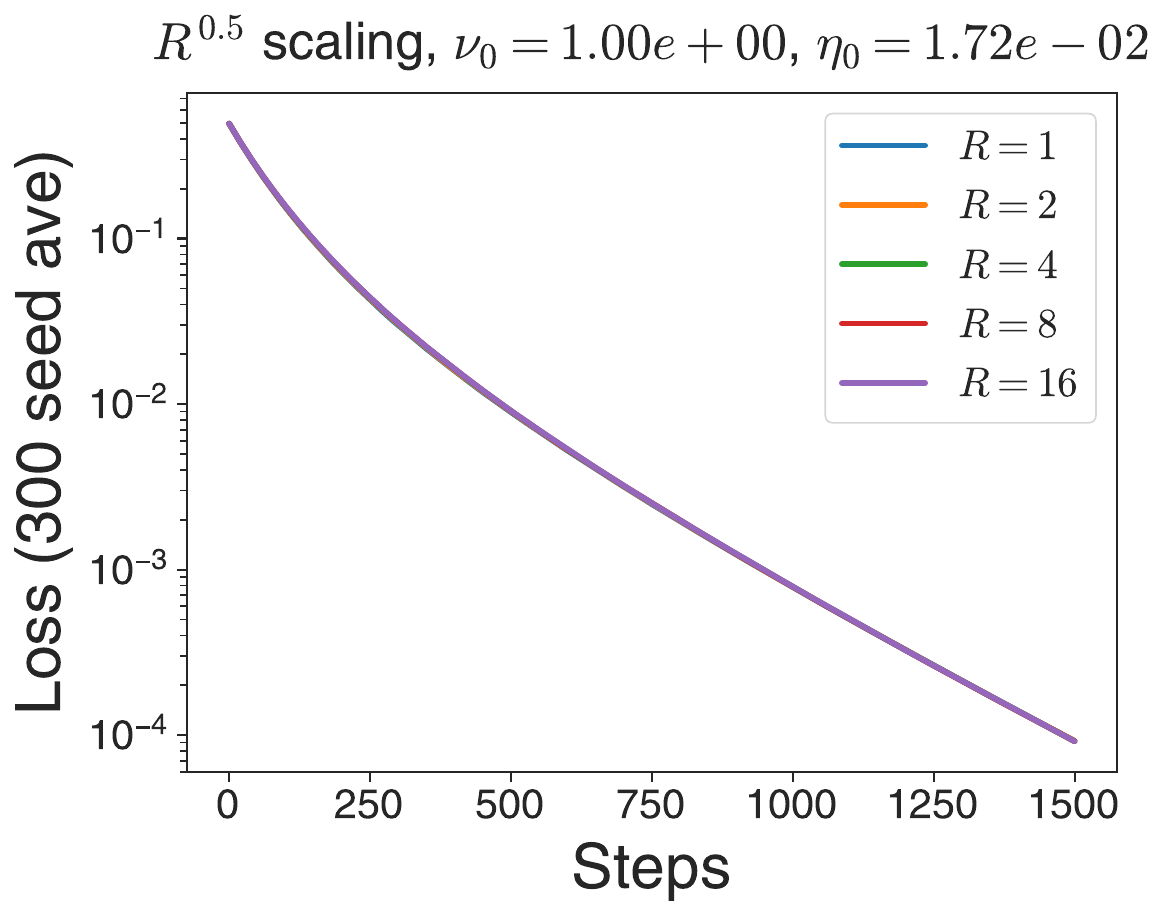} &
\includegraphics[width=0.3\linewidth]{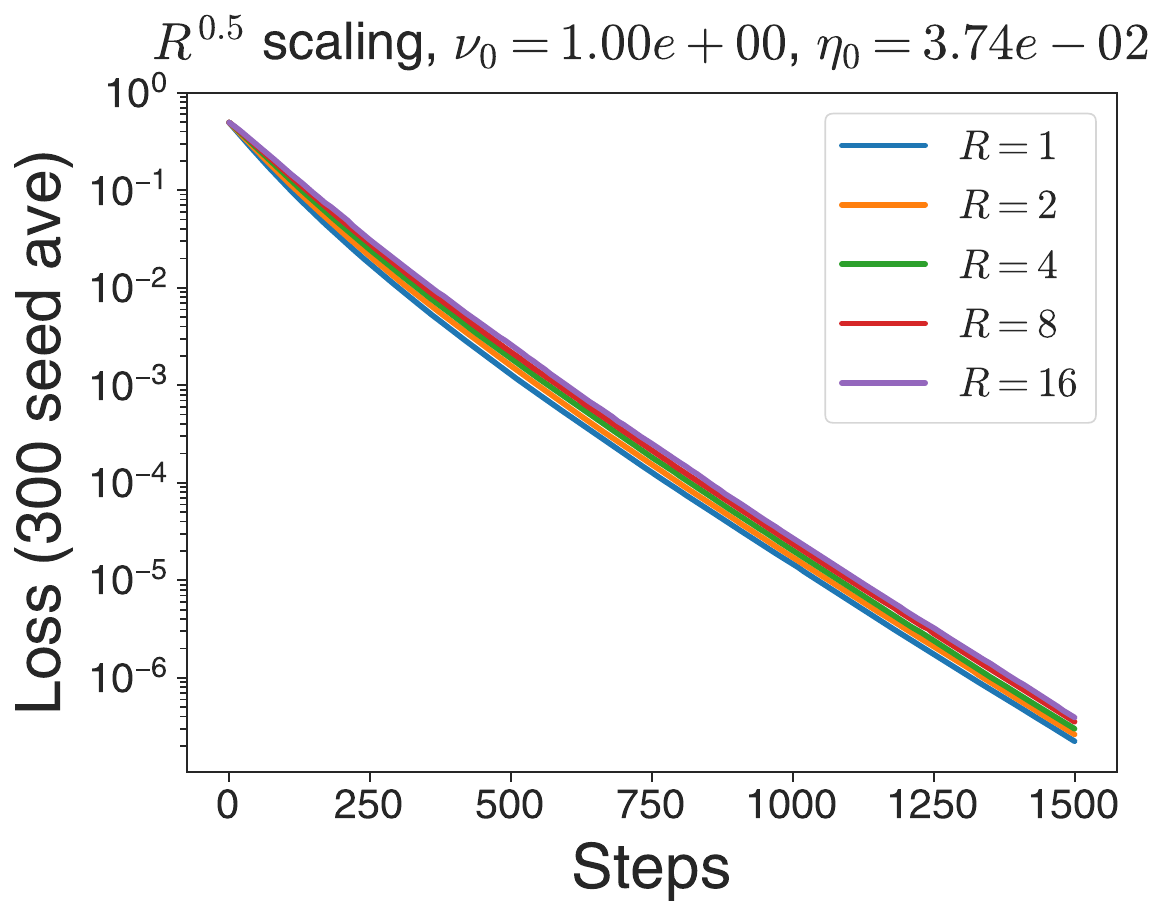}
\end{tabular}
\caption{\aga{Possible appendix}Averaged loss curves for \la diloco (solid lines) are well captured by theoretical model (dashed lines) across many settings with fixed total batch size (all panels, $\D = 3200$, $\B_{tot} = 64$, various learning rates). Keeping $\lr$ and $\lafrac$ independent of
number of workers $\R$ leads to similar dynamics at smaller learning rates (top left) but causes learning curves for larger $\R$ to diverge early at larger learning rates (top right). $\R^{-1/2}$ scaling rule for $\lr$, with $\lafrac\lr$ fixed gives better correspondence
across $\R$ for all learning rates (bottom), but dynamics still becomes non-universal at larger $\lr$ (bottom right)}
\label{fig:r_scaling_main}
\end{figure}

\section{Deterministic dynamics of \sla}

In this section we derive the deterministic dynamics of \sla in the linear setting, and
show how the dynamics are different from weight averaging and GPA.

\subsection{SLA with EMA momentum}

\label{app:sla_ema_def}

The update rules for EMA style momentum in the inner and outer loop both are given by:
\begin{equation}
\psb_{t} := \th_{t}{\rm~(when~} t//\TT=0)
\end{equation}
\begin{equation}
\left.
\begin{aligned}
\tmm_{t+1} &= \bin \tmm_{t}+(1-\bin)\g(\psb_{t}) \\
\psb_{t+1} &= \psb_{t}-\lr \tmm_{t+1}
\end{aligned}
\right\} \quad {\rm~for~}\TT{\rm~steps}
\end{equation}
\begin{equation}
\mm_{t+\TT} = \bout \mm_{t}+(1-\bout)(\psb_{t+\TT}-\psb_{t})
\end{equation}
\begin{equation}
\th_{t+T} = \th_{t}- \lafrac \mm_{t+\TT}
\end{equation}

Note that the inner momentum $\tmm$ is not typically synced across cycles.

\subsection{SLA with Nesterov momentum}

\label{app:sla_nest_def}

The update rules for Nesterov style momentum in the inner and outer loop both are given by:
\begin{equation}
\psb_{t} := \th_{t}{\rm~(when~} t//\TT=0)
\end{equation}
\begin{equation}
\left.
\begin{aligned}
\tmm_{t+1} &= \bin \tmm_{t}+(1-\bin)\g(\psb_{t}) \\
\psb_{t+1} &= \psb_{t}-\lr [(1-\bin)\g(\psb_{t})+\bin\tmm_{t+1}]
\end{aligned}
\right\} \quad {\rm~for~}\TT{\rm~steps}
\end{equation}
\begin{equation}
\mm_{t+\TT} = \bout \mm_{t}+(1-\bout)(\psb_{t+\TT}-\psb_{t})
\end{equation}
\begin{equation}
\th_{t+T} = \th_{t}- \lafrac [(1-\bout)[\psb_{t+\TT}-\psb_{t}]+\bout\mm_{t+\TT}]
\end{equation}

\subsection{Full linear dynamics}

\label{app:sla_det_dyn_full}

In the deterministic setting of the linear regression problem (full batch size), we can write the \sla dynamics in function space as:
\begin{equation}
\begin{pmatrix}
\tz_{t+\TT} \\
\tk_{t+\TT} \\
\z_{t+\TT} \\
\kvec_{t+\TT}
\end{pmatrix} = 
\Tsyn \Tout \Tin \begin{pmatrix}
\tz_{t} \\
\tk_{t} \\
\z_{t} \\
\kvec_{t}
\end{pmatrix}
\end{equation}
Here $\Tin$ and $\Tout$ represent the inner and outer optimizers respectively, and $\Tsyn$ represents the ``syncing'' operation that prepares $\z$ and $\kvec$ for the next cycle.

The standard definition of \sla which only syncs $\z$ has the syncing operator:
\begin{equation}
\Tsyn = \begin{pmatrix}
0 & 0 & \Id & 0\\
0 & \Id & 0 & 0 \\
0 & 0 & \Id & 0 \\
0 & 0 & 0 & \Id
\end{pmatrix}
\end{equation}
while the $2$-D simplification where $\kvec_{t}$ is initialized to $0$ every cycle has the operator:
\begin{equation}
\Tsyn = 
\begin{pmatrix}
0 & 0 & \Id & 0\\
0 & 0 & 0 & 0 \\
0 & 0 & \Id & 0 \\
0 & 0 & 0 & \Id
\end{pmatrix}
\end{equation}
For EMA style momentum, the inner and outer optimizer are given by the operators
\begin{equation}
\Tin = \begin{pmatrix}
\Id-\lr(1-\bin)\frac{1}{D}\ntk & -\eta\bin\Id & 0 & 0\\
(1-\bin)\frac{1}{D}\ntk & \bin\Id & 0 & 0\\
0 & 0 & \Id & 0 \\
0 & 0 & 0 & \Id
\end{pmatrix}^{\TT},~\Tout = \begin{pmatrix}
\Id & 0 & 0 & 0\\
0 & \Id & 0 & 0 \\
\lafrac(1-\bout)\Id & 0 & \Id-\lafrac(1-\bout)\Id & -\lafrac\bout\Id \\
-(1-\bout)\Id & 0 & (1-\bout)\Id & \bout\Id \\
\end{pmatrix}
\end{equation}
while for Nesterov (in EMA units) we have
\begin{equation}
\Tin = \begin{pmatrix}
\Id-\lr(1-\bin^{2})\frac{1}{D}\ntk & -\eta\bin^{2}\Id & 0 & 0\\
(1-\bin)\frac{1}{D}\ntk & \bin\Id & 0 & 0\\
0 & 0 & \Id & 0 \\
0 & 0 & 0 & \Id
\end{pmatrix}^{\TT},~\Tout = \begin{pmatrix}
\Id & 0 & 0 & 0\\
0 & \Id & 0 & 0 \\
\lafrac(1-\bout^{2})\Id & 0 & \Id-\lafrac(1-\bout^{2})\Id & -\lafrac\bout^{2}\Id \\
-(1-\bout)\Id & 0 & (1-\bout)\Id & \bout\Id \\
\end{pmatrix}
\end{equation}

We can once again shift to the eigenbasis of $\frac{1}{\D}\ntk$. This reduces the dynamics to a $4$-dimensional system for each eigenmode. For EMA momentum, for example, we have:
\begin{equation}
\Tin = \begin{pmatrix}
1-(1-\bin)\lr\lam & -\eta\bin & 0 & 0\\
(1-\bin)\lam & \bin & 0 & 0\\
0 & 0 & 1 & 0 \\
0 & 0 & 0 & 1
\end{pmatrix}^{\TT},~\Tout = \begin{pmatrix}
1 & 0 & 0 & 0\\
0 & 1 & 0 & 0 \\
\lafrac(1-\bout) & 0 & 1-\lafrac(1-\bout) & -\lafrac\bout \\
-(1-\bout) & 0 & (1-\bout) & \bout \\
\end{pmatrix}
\end{equation}

The inner and outer optimizers have a block diagonal structure:
\begin{equation}
\Tin = \begin{pmatrix}
\mF & 0\\
0 & \Id
\end{pmatrix},~\Tout = \begin{pmatrix}
\Id & 0 \\
\mG & \mH
\end{pmatrix}, \Tout\Tin = \begin{pmatrix}
\mF & 0 \\
\mG\mF & \mH
\end{pmatrix}
\end{equation}
which for EMA momentum gives
\begin{equation}
\mF = \begin{pmatrix}
1-(1-\bin)\lr\lam & -\eta\bin \\
(1-\bin)\lam & \bin
\end{pmatrix}^{\TT},~\G = \begin{pmatrix}
\lafrac(1-\bout) & 0  \\
-(1-\bout) & 0 
\end{pmatrix},~\H = \begin{pmatrix}
 1-\lafrac(1-\bout) & -\lafrac\bout \\
 (1-\bout) & \bout
\end{pmatrix}
\end{equation}
The $\mF$ matrix is of particular interest, as it defines the ``effective spectrum'' seen by the outer optimizer in the linear regression setting.

\subsection{Momentum with weight averaging}

\label{app:weight_ema_vs_sla}

Our analysis of \sla in the deterministic linear regression setting begs the immediate question: are there algorithms which don't have the two-phase
structure, but have similar behavior after $\TT$ steps? While this question is still generally open, we can answer in the negative for the most basic
idea: pairing the dynamics of SGD momentum with an exponential moving average of the parameters.

Consider the deterministic setting of the linear regression problem.
In the diagonalized basis,
given an eval point $\zema$ which is an EMA of $\zz$ with coefficient $\bema$ (where
$\bema = 0$ corresponds to normal evaluation at the latest parameter value), the update
rule for the $3$-dimensional system becomes:
\begin{equation}
\begin{pmatrix}
\zz_{t+1}\\
\kk_{t+1} \\
\zema_{t+1}
\end{pmatrix} =
\begin{pmatrix}
1 & 0 & 0\\
0 & 1 & 0\\
1-\bema & 0 & \bema
\end{pmatrix}
\begin{pmatrix}
1-\lr(1-\bone)\lam & -\lr\bone & 0\\
(1-\bone)\lam & \bone & 0 \\
0 & 0 & 1
\end{pmatrix}\begin{pmatrix}
\zz_{t}\\
\kk_{t} \\
\zema_{t}
\end{pmatrix}
\end{equation}
Evaluating we have:
\begin{equation}
\begin{pmatrix}
\zz_{t+1}\\
\kk_{t+1} \\
\zema_{t+1}
\end{pmatrix} =
\begin{pmatrix}
1-\lr(1-\bone)\lam & -\lr\bone & 0\\
(1-\bone)\lam & \bone & 0 \\
(1-\bema)[1-\lr(1-\bone)] & -(1-\bema)\lr\bone & \bema
\end{pmatrix}\begin{pmatrix}
\zz_{t}\\
\kk_{t} \\
\zema_{t}
\end{pmatrix}
\end{equation}
Since $\zema$ depends on $\zz$ and $\kk$, but does not causally affect them, two of the
eigenvalues of this system are the eigenvalues of the original system. This can be seen
directly by the fact that there are two left eigenvectors of the transition matrix
which lie in the $\zz$/$\kk$ subspace alone. These eigenvectors have the same
eigenvalues as the original problem. The remaining eigenvalue has eigenvalue $\bema$,
since the trace of the new system is the trace of the old system plus $\bema$.

This means that EMA of the evaluation point does not appreciably change the
eigenstructure of the problem in the same way that SLA does via the outer momentum
applied to the ``effective'' problem.

\subsection{GPA in linear regression}

\label{app:gpa_vs_sla}

The GPA algorithm in \citet{defazio2025smoothing} (and indeed, the paper title) is partially inspired by coming up with and alternative of single-worker \diloco which
maintains its benefits, but doesn't require the very discrete structure of the two-phase optimization process.
The paper suggests that GPA can be thought of as a continuous relaxation of \diloco. Here we show that in the deterministic
setting, \sla and GPA have quite different dynamics.

The original GPA formulation can be defined by the following equations:
\begin{equation}
\yy_{t} = \muy \xx_{t}+(1-\muy)\zz_{t}
\end{equation}
\begin{equation}
\zz_{t+1} = \zz_{t}+\lr_{t}\dd(\yy_{t})
\end{equation}
\begin{equation}
\xx_{t+1} = \mux \xx_{t}+(1-\mux)\zz_{t+1}
\end{equation}
Here $\dd(\yy_{t})$ represents the ``update'' of a base optimizer - for example, the negative gradient in the case of SGD. The variable $\xx_{t}$ is used for evaluation.

To understand the dynamics, we first re-write everything so that we have equations for
$\yy_{t+1}$, $\zz_{t+1}$, and $\xx_{t+1}$ as functions of $\yy_{t}$, $\zz_{t}$, and
$\xx_{t}$ only. This gives us:
\begin{equation}
\yy_{t+1} =  \mux \xx_{t+1}+(1-\muy)\zz_{t+1} = \muy\mux \xx_{t}+\muy(1-\mux)\zz_{t+1}+(1-\muy)\zz_{t+1}
\end{equation}
\begin{equation}
\yy_{t+1} =   \mux\muy \xx_{t}+(1-\mux\muy)[\zz_{t}+\lr_{t}\dd(\yy_{t})]
\end{equation}

The corresponding $\xx$ and $\zz$ equations are simple to compute, and we end up with:
\begin{equation}
\begin{pmatrix}
\yy_{t+1} \\
\zz_{t+1} \\
\xx_{t+1}
\end{pmatrix} = 
\begin{pmatrix}
0 & 1-\mux\muy & \mux\muy \\
0 & 1 & 0 \\
0 & 1-\mux & \mux
\end{pmatrix}
\begin{pmatrix}
\yy_{t} \\
\zz_{t} \\
\xx_{t}
\end{pmatrix}+\lr_{t}\dd(\yy_{t})\begin{pmatrix}
1-\mux\muy \\
1 \\
1-\mux
\end{pmatrix}
\end{equation}
In the deterministic setting of the linear regression model, the optimization problem
decomposes over the eigenmodes and we have $\dd(\yy_{t}) = -\lam\yy_{t}$. This gives us:
\begin{equation}
\begin{pmatrix}
\yy_{t+1} \\
\zz_{t+1} \\
\xx_{t+1}
\end{pmatrix} = 
\begin{pmatrix}
-\lr_{t}\lam(1-\mux\muy ) & 1-\mux\muy & \mux\muy \\
-\lr_{t}\lam & 1 & 0 \\
-(1-\mux)\lr_{t}\lam & 1-\mux & \mux
\end{pmatrix}
\begin{pmatrix}
\yy_{t} \\
\zz_{t} \\
\xx_{t}
\end{pmatrix}
\end{equation}
This has very different properties from any of the \sla matrices we analyzed in
Section \ref{sec:sla}, particularly in how $\lam$, $\lr$, and the other hyperparameters fold in to the overall spectrum. There is no
clear ``effective problem'' dominating the eigenstructure like in \sla.

Another way we can see this is by changing variables. We define:
\begin{equation}
\hkk_{t}\equiv \frac{1}{\mux\muy} \left(-\yy_{t}+(1-\mux\muy)\zz_{t}+\mux\muy\xx_{t}\right)
\end{equation}
Then the dynamics of $\yy_{t}$ and $\hkk_{t}$ close and we have:
\begin{equation}
\hkk_{t+1} = \mux\hkk_{t}+(1-\mux)\lr_{t}\dd(\yy_{t})
\label{eq:gpa_mom_simple}
\end{equation}
\begin{equation}
\yy_{t+1} = \yy_{t}+\muy\hkk_{t+1}+(1-\muy)\lr_{t}\dd(\yy_{t})
\label{eq:gpa_y_simple}
\end{equation}
We can see that this is very similar to Nesterov momentum, with an additional hyperparameter that decouples the contribution
of the momentum and gradient to the parameter update. This also will have very
different eigenstructure from \sla with Nesterov outer; in particular, it misses the interaction between the outer momentum and
highly convergent eigenmodes seen in \sla. This suggests that the two algorithms are indeed qualitatively different.

\section{Additional Experiments}

\subsection{Single cycle advantage of \la vs. SGD}
\label{app:la_faster_numerics}

In order to study the finite-time advantage of \la,
we numerically iterated Equation \ref{eq:look_ahead_pvec} for a variety of spectra
$\lmat$, batch fraction $\B/\D$, and sync cycle period $\TT$.
For each such setting, we ran $\TT$ steps of dynamics over a grid of $\lr$
and $\lafrac$ values to calculate the change in loss after one cycle.
As expected, we generally found that the optimal value of $\lr$ for a particular $\lafrac$
decreased in $\lafrac$ (Figure \ref{fig:grad_spike_heatmap}, left).

We found that there were many settings for which the optimal $(\lafrac, \lr)$ pair
involved $\lafrac>1$, as it does at infinite time.
In general, variance in the spectrum causes the optimal $\lafrac$ to
shift away from $1$. For example, consider a ``spiked spectrum'' with two
eigenvalues $\lambda_{0}$ ($99\%$ of eigenmodes) and $20\lambda_{0}$ ($1\%$ of eigenmodes).
Plotting the loss after one cycle as a function of $\lafrac$
(optimizing $\lr$ for each $\lafrac$), we see that the optimal $\lafrac$
is greater than $1$ (Figure \ref{fig:grad_spike_heatmap}, middle).

For a single cycle, we also found some settings where the optimal $\lafrac<1$,
such as certain power law spectra (Figure \ref{fig:grad_spike_heatmap}, right, $\alpha = -1.5$).
We note that the total benefit of \la amounts to a $10-20\%$ improvement in the loss decrease
per cycle in the settings we present here. We conjecture that \la can provide modest but bounded gains in
more general settings. The largest benefits tend to occur for smaller batch fraction $\B/\D$ and longer cycles $\TT$.

\begin{figure*}[th]
\begin{tabular}{ccc}
\includegraphics[height=0.25\linewidth]{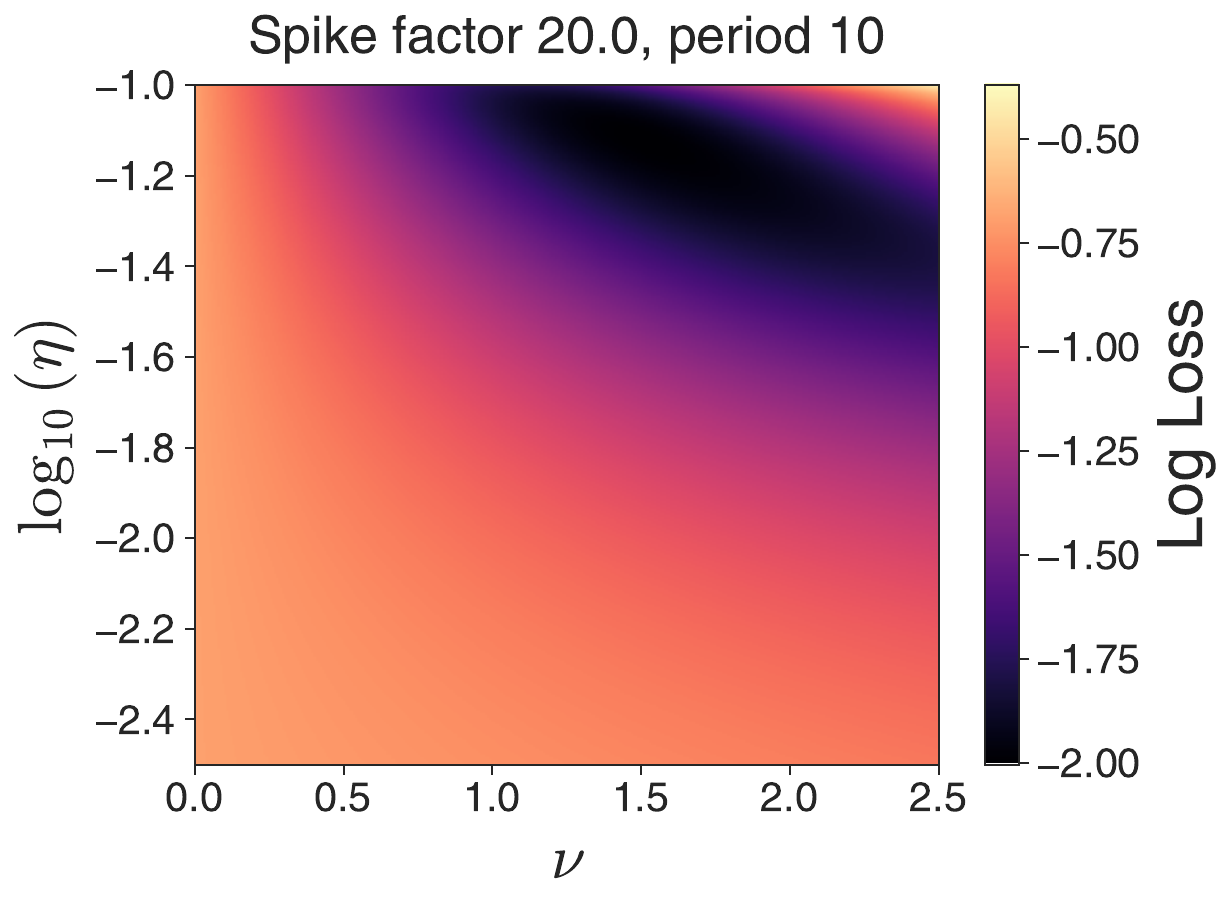} & \includegraphics[height=0.22\linewidth]{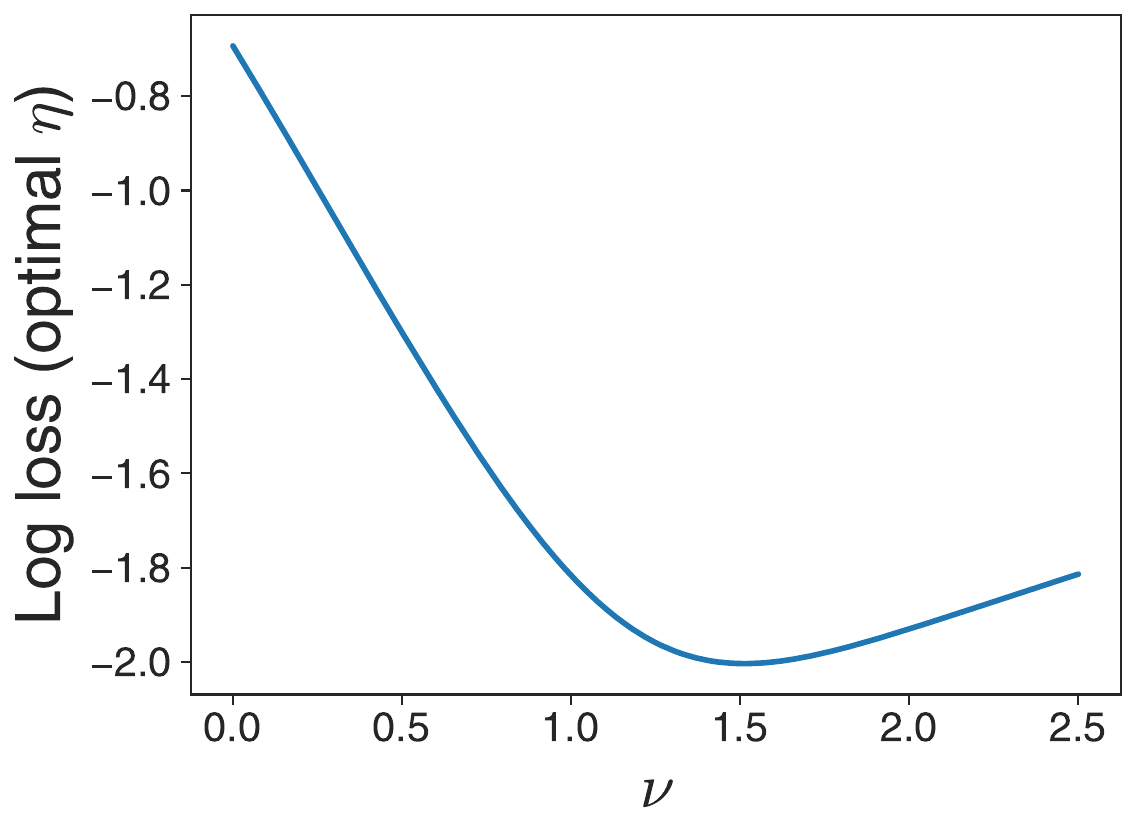} &
\includegraphics[height=0.22\linewidth]{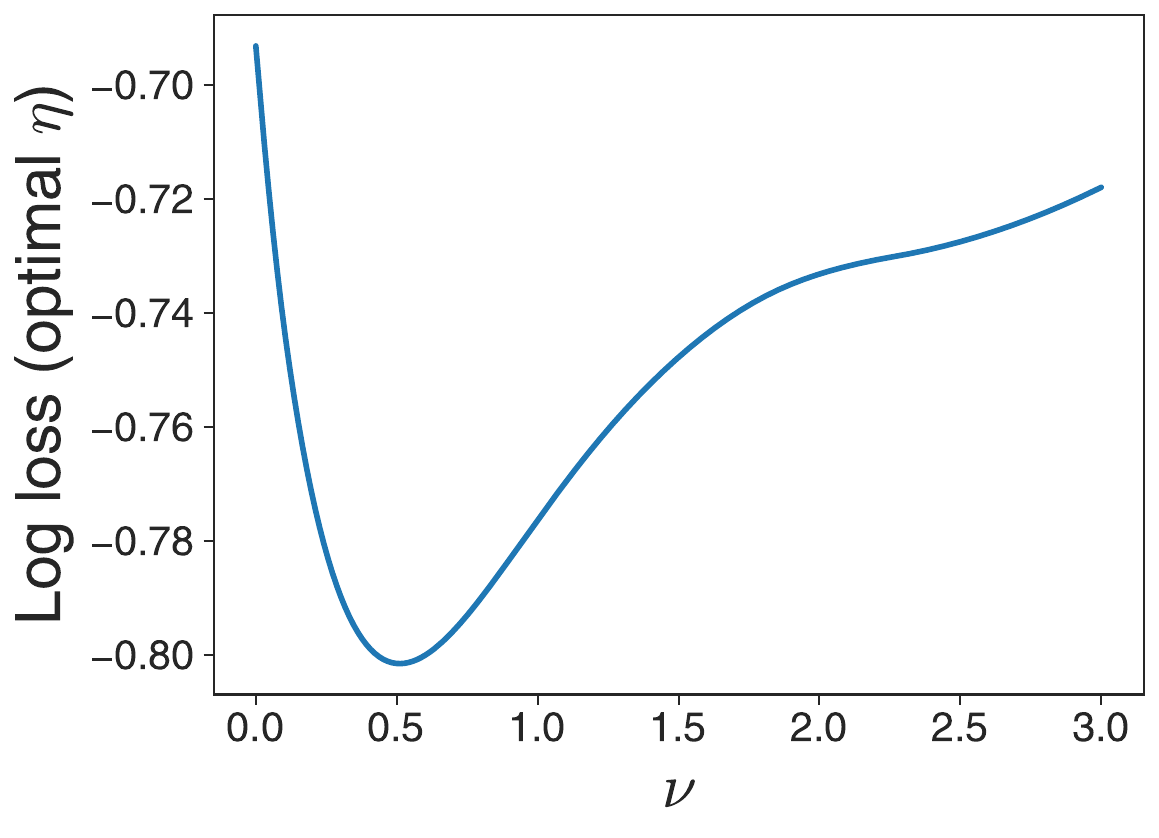}
\end{tabular}
\caption{Loss after one cycle of $\TT = 10$ steps as a function of $\lr$ and $\lafrac$
for spiked spectrum (left). Batch fraction $\B/\D = 0.2$, and spectrum has two
eigenvalues $\lambda_{0}$ ($99\%$ of eigenmodes) and $20\lambda_{0}$ ($1\%$ of eigenmodes). Optimal $\lr^*(\lafrac)$ decreases in $\lafrac$. Loss at $\lr^*(\lafrac)$ is minimized for $\lafrac >1$ (middle). Power law spectrum with exponent $\alpha = -1.5$ has optimal $\lafrac<1$ (right, $\B/\D = 0.005$).}
\label{fig:grad_spike_heatmap}
\end{figure*}

\subsection{$\R^{-1/2}$ scaling in a non-linear setting}
\label{app:nonlinear_scaling}

Our theoretical results, including the predicted $R^{-1/2}$ scaling rule in the
fixed total batch size setting, were derived for linear models. Some version
of the results holds in the NTK limit where learning is linearized. There is an
open question as to what extent the results hold in the non-linear setting.

As a first step to answer this question, we conducted experiments training a
two-layer MLP on MNIST using LA-Diloco
(batch size $256$, hidden dimension $128$, cycle length $\TT=10$).
We first trained with $\R=1$ and optimized the inner learning rate.
We then varied $\R$ and set the learning rates
in two ways: first, by leaving both learning rates fixed,
and second by using the proposed $\R^{-1/2}$ scaling rule for the inner optimizer
(and counter-scaling the outer learning rate $\lafrac$).

The results can be seen in figure \ref{fig:two_layer_loss}.
We see that there is no collapse in learning curves when the inner and outer
learning rates are kept fixed across $\R$ (left plot). When we do scale
the inner learning rate (right), we have very good agreement with two exceptions:
the early time dynamics at $\R=2$ does not collapse
(though the later time dynamics do), and for large $\R$ the training
becomes unstable (since $\lafrac>1$ using the scaling rule).
We suspect that the early-time differences are due to non-linear
effects.

\begin{figure}[h]
\centering
\begin{tabular}{c}
\includegraphics[width=0.9\textwidth]{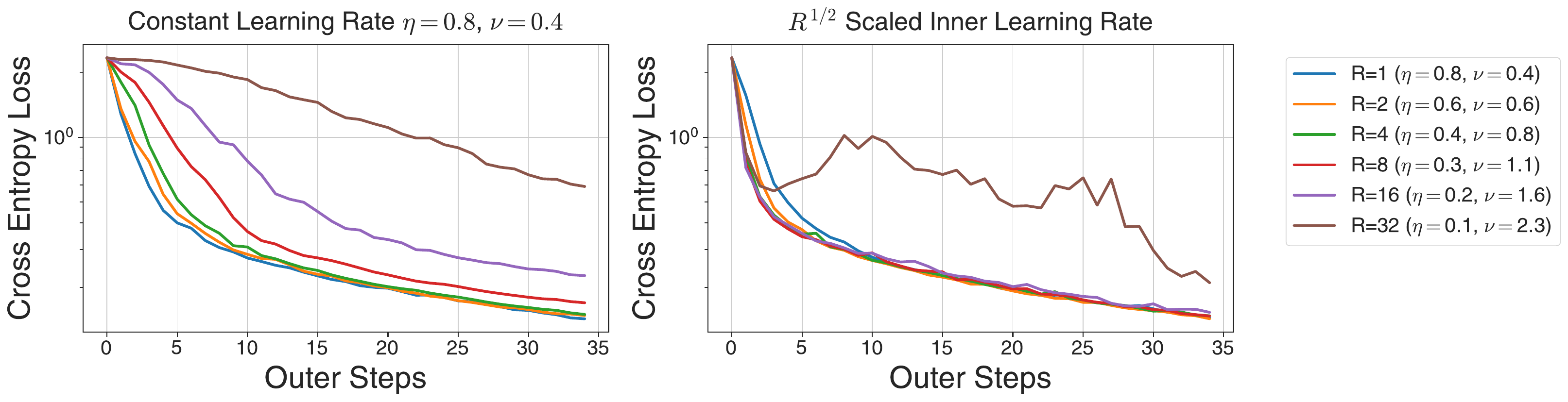}
\end{tabular}
\caption{Loss curves for a two-layer MLP on MNIST, trained using LA-Diloco with various
$\R$. When the inner and outer learning rates are kept
fixed as a function of $\R$ (left), we see no collapse.
When we use the proposed scaling rule (right), we see good agreement
except at early times or for higher values of $\R$.
}
\label{fig:two_layer_loss}
\end{figure}

The original motivation for the scaling rule was based on an analysis of the
optimization noise. In the non-linear setting, we can analyze the noise via the
per-worker dispersion $\expect[||\psi_{t, k}-\bar{\psi}_{t}||^{2}]$, which
measures the deviation between the workers at the end of the cycle. Here we
take an empirical average over the workers, but also average over the randomness
in the optimization trajectory.

We measured this quantity in the same MNIST setting, by replaying a single cycle
$100$ times at various training checkpoints (re-randomizing the data sampling
each time). When the learning rate was not scaled with $\R^{-1/2}$, this
deviation grows roughly linearly with $\R$ (Figure \ref{fig:two_layer_dispersion}, left).
With the scaling rule, the deviation still grows with $\R$ but at a much
slower rate (Figure \ref{fig:two_layer_dispersion}, right). This suggests that
the scaling rule does indeed control the between-worker deviation, but that the
scaling does not fully normalize the noise --- which is true even in the linear
theory due to the higher order terms.

\begin{figure}[h]
\centering
\begin{tabular}{c}
\includegraphics[width=0.9\textwidth]{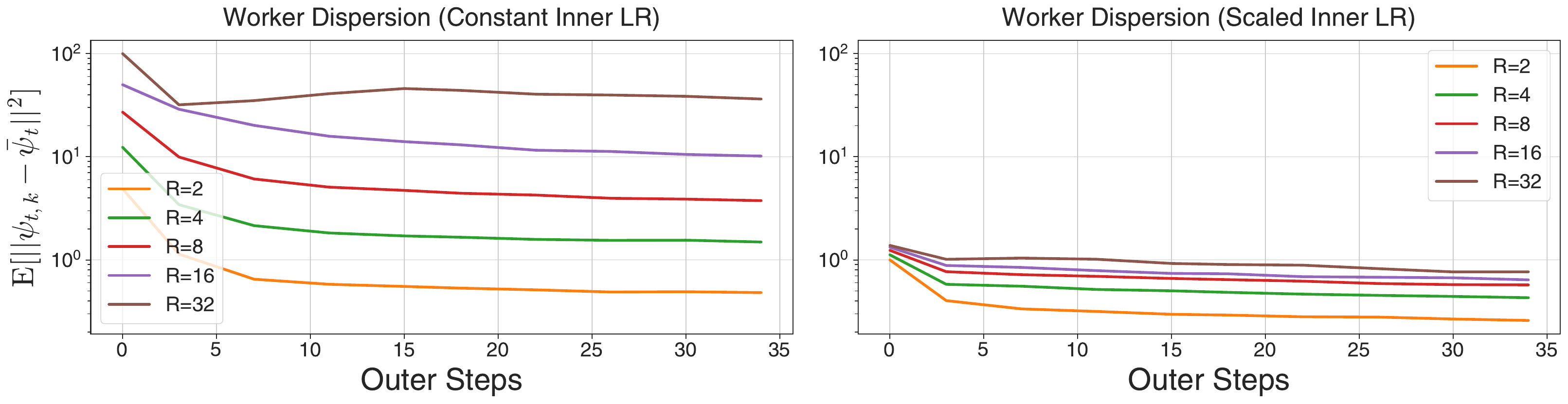}
\end{tabular}
\caption{Per-worker dispersion $\expect[||\psi_{t, k}-\theta_{t}||^{2}]$ for a two-layer MLP on MNIST, trained using LA-Diloco with various $\R$.
When the inner and outer learning rates are kept fixed as a function of $\R$, we see a large increase in dispersion.
When we use the proposed scaling rule, the dispersions are similar across $\R$, reflecting a more similar noise process.
}
\label{fig:two_layer_dispersion}
\end{figure}

These results suggest that the scaling rule may be beneficial in the
non-linear setting. More work is needed to establish this as a good strategy
in practice; however both the scaling itself and the underlying noise analysis
are worth further study.

\subsection{SLA momentum variations on deep learning workloads}
\label{app:algoperf_experiments}

In order to provide additional evidence that outer loop Nesterov is often favored over outer loop heavy ball momentum, we tuned SLA with
different momentum variations on the FineWebEDU and Imagenet-ViT workloads from AlgoPerf \citep{Dahl2023AlgoPerf}. We followed the tuning methodology in \citet{Kasimbeg2025AlgoPerfResults}, using a $50$-point
Halton search for each workload-momentum variant pair. The inner optimizer was AdamW, and the Halton points were chosen to cover the same
hyperparameters for each workload. The best runs for each workload-algorithm pair can be
found in Table \ref{tab:algoperf_sla}.

\begin{table}[h]
\centering
\begin{tabular}{lllccll}
Workload & Algorithm & Validation Metric & $\eta\lafrac$ & $\lafrac$ & $\bout$ & $\bin$ \\
\hline
fineweb\_edu\_lm & Nesterov & Perplexity: 28.5 & 0.000785 & 0.866 & 0.832 & 0.804 \\
fineweb\_edu\_lm & Heavy Ball & Perplexity: 29.3 & 0.000785 & 0.866 & 0.832 & 0.804 \\
imagenet\_vit\_s16 & Nesterov & Error Rate: 0.265 & 0.00189 & 1.29 & 0.817 & 0.930 \\
imagenet\_vit\_s16 & Heavy Ball & Error Rate: 0.269 & 0.00189 & 1.29 & 0.817 & 0.930
\end{tabular}
\caption{Validation metrics and key hyperparameters for best training runs with SLA on the AlgoPerf benchmark ($50$ point Halton search).
Nesterov momentum appears to perform better than heavy ball momentum in this setting.}
\label{tab:algoperf_sla}
\end{table}

We found that the two outer optimizers behaved similarly,
but there was a slight advantage to Nesterov in the outer optimizer;
this advantage was small but consistent across the best performing configurations.
The best hyperparameters in our search are equal for both momentum methods because
the most important parameter is the total learning rate $\lafrac\lr$, followed by
the outer learning rate $\lafrac$.

\newpage


\end{document}